\pdfoutput=1

\def\extended{true}

\PassOptionsToPackage{usenames,dvipsnames,table}{xcolor}

\documentclass[a4paper,twoside,english]{article}

\twocolumn
\usepackage[T1]{fontenc}
\usepackage[utf8]{inputenc}
\usepackage{babel}
\usepackage{times}

\usepackage{microtype}
\usepackage{graphicx}
\usepackage{subfigure}
\usepackage{booktabs} 

\usepackage[hidelinks]{hyperref}

\ifdefined\extended
\newcommand{\extref}[1]{\ref{#1}}
\else
\usepackage{xr}
\newcommand{\extref}[1]{\ref{ext-#1}}
\fi

\usepackage[utf8]{inputenc}

\usepackage{xspace}
\graphicspath{{./figures/}{./}{../JFPDA-18/figures/}{./JFPDA-18/}}

\usepackage{enumitem}

\usepackage{rotating}
\usepackage{pifont}

\usepackage[normalem]{ulem}
\usepackage{verbatim}
\usepackage[most]{tcolorbox}

\usepackage{scalefnt}
\usepackage{amsmath} \makeatletter
\DeclareRobustCommand*\cal{\@fontswitch\relax\mathcal}
\makeatother
\usepackage{amsthm}
\usepackage{amssymb}
\usepackage{bbold} \usepackage{mathtools}
\usepackage{empheq}

\usepackage{xpatch}
\usepackage{thmtools}
\usepackage{apptools}
\makeatletter
\xpatchcmd{\thmt@restatable}{\csname #2\@xa\endcsname\ifx\@nx#1\@nx\else[{#1}]\fi}{\IfAppendix{\csname #2\@xa\endcsname}{\csname #2\@xa\endcsname\ifx\@nx#1\@nx\else[{#1}]\fi}}{}{} \makeatother

\allowdisplaybreaks
\usepackage{centernot}

\DeclarePairedDelimiter\ceil{\lceil}{\rceil}

\newcommand{\twodots}{\mathinner {\ldotp \ldotp}}
\usepackage{xspace}
\def\ie{{\em i.e.}\xspace}
\def\eg{{\em e.g.}\xspace}
\def\cf{{\em cf.}\xspace}

\def\reals{{\mathbb R}}
\def\cS{{\cal S}}
\def\cA{{\cal A}}
\def\cB{{\cal B}}

\def\cZ{{\cal Z}}
\def\occ{o}

\def\Occ{O}
\newcommand{\PP}[4]{P_{#2}^{#4}(#3|#1)}
\def\cH{{\cal H}}
\def\cK{{\cal K}}

\def\l{\lambda}
\def\p{1} \def\WUL{W} \def\va{{\boldsymbol{a}}}

\def\vz{{\boldsymbol{z}}}
\def\vth{{\boldsymbol{\theta}}}
\def\vTh{{\boldsymbol{\Theta}}}
\def\vpi{{\boldsymbol{\pi}}}
\def\vbeta{{\boldsymbol{\beta}}}
\def\vmu{{\boldsymbol{\mu}}}
\def\vx{{\boldsymbol{x}}}
\def\vy{{\boldsymbol{y}}}

\def\nev{\textsc{nev}} \def\nes{\textsc{nes}}  \def\nxt{T}

\def\width{\mathit{width}}

\def\depth{\tau}

\def\radius{\rho}
\def\thr{thr}

 \usepackage[ruled,vlined,linesnumbered]{algorithm2e}
\SetKwProg{Fct}{Fct}{}{}

\newtheorem{theorem}{Theorem}
\newtheorem{lemma}{Lemma}
\newtheorem{corollary}{Corollary}

\newcommand{\eqdef}     {\stackrel{{\textrm{\rm\tiny def}}}{=}}

\ifdefined\noproofs
\usepackage{environ}
\NewEnviron{killcontents}{}

\else{}\fi

\makeatletter \newcommand{\pushright}[1]{\ifmeasuring@#1\else\omit\hfill$\displaystyle#1$\fi\ignorespaces}
\newcommand{\pushleft}[1]{\ifmeasuring@#1\else\omit$\displaystyle#1$\hfill\fi\ignorespaces}
\newcommand{\specialcell}[1]{\ifmeasuring@#1\else\omit$\displaystyle#1$\ignorespaces\fi}
\makeatother

\DeclarePairedDelimiter{\abs}{\lvert}{\rvert}\DeclarePairedDelimiter{\norm}{\lVert}{\rVert}

\def\player{}
\def\Player{}

\usepackage{pgf} \usepackage{tikz}
\usetikzlibrary{trees}

\newcommand{\ifextended}[2]{\ifdefined\extended{#1}\else{#2}\fi}

\newcommand{\tcr}[1]{\textcolor{red}{#1}}

\begin{document}

\title{\Large\bf
  On Bellman's Optimality Principle for zs-POSGs}

\hypersetup{pdfauthor={Olivier Buffet, Jilles Dibangoye, Aurélien Delage, Abdallah Saffidine, Vincent Thomas},
  pdftitle={On Bellman's Optimality Principle for zs-POSGs}}

\author{  Olivier Buffet${}^1$ \hfill Jilles Dibangoye${}^2$ \hfill Aurélien Delage${}^{1,2}$ \hfill Abdallah Saffidine${}^3$  \hfill Vincent Thomas${}^1$ \\
 \\
${}^1$        Université de Lorraine, CNRS, INRIA, LORIA, F-54000 Nancy, France \\
${}^2$        Univ. Lyon, INSA Lyon, INRIA, CITI, F-69621 Villeurbanne, France \\
${}^3$        The University of New South Wales, Sydney, Australia \\
 \\
{\tt (firstname.lastname@inria.fr$|$abdallahs[@]cse.unsw.edu.au)}
}

\date{}

\maketitle
\thispagestyle{empty}

\fbox{
  \begin{minipage}{0.95\columnwidth}
    \uline{Warning:}
The work presented in this paper allows computing the $\epsilon$-optimal (Nash equilibrium) value function of a zero-sume POSG, but does not discuss how to derive a safe (aka globally consistent) strategy for a player, \ie, one that the opponent cannot exploit.
This issue has been addressed in subsequent research, which also presents several other improvements \cite{DelBufDib-corr21,DelBufDibAbd-corr22}.
See preferably \cite{DelBufDibAbd-corr22}.
    \end{minipage}
}

\subsection*{Abstract}
{\em Many non-trivial sequential decision-making problems are efficiently
  solved by relying on Bellman's optimality principle, \ie, exploiting
  the fact that sub-problems are nested recursively within the
  original problem.
Here we show how it can apply to (infinite horizon) 2-player
  zero-sum partially observable stochastic games (zs-POSGs) by
(i) taking a central planner's viewpoint, which can only reason on a
  sufficient statistic called occupancy state, and
(ii) turning such problems into zero-sum occupancy Markov games (zs-OMGs).
Then, exploiting the Lipschitz-continuity of the value function in
  occupancy space, one can derive a version of the HSVI algorithm
  (Heuristic Search Value Iteration) that provably finds an
  $\epsilon$-Nash equilibrium in finite time.
}
\subsection*{Keywords}
POSG; partially observable stochastic game; Bellman's optimality
principle; Heuristic Search Value Iteration.

\section{Introduction}

Bellman's optimality principle (Bellman's OP)
\cite{Bellman-pnas52} led to state-of-the-art solvers in many
non-trivial sequential decision-making problems, assuming
partial observability \cite{Astrom-jmaa65},
multi-objective criteria \cite{SteWhi-jacm91,Machuca-ijcai11},
collaborating agents, \eg, modeled as decentralized partially
observable Markov decision processes (Dec-POMDPs)
\cite{HanZil-aaai04,SzeChaZil-uai05,DibAmaBufCha-jair16}, or
some non-collaborative perfect information games (from Shapley's seminal work
\cite{Shapley-pnas53} to \cite{BufDibSafTho-tog20}).
In all these settings this principle exploits the recursive nesting of
sub-problems within the original problem.
An open question is whether---and how---it could be applied to
imperfect information games, which are encountered in diverse
applications such as Poker \cite{Kuhn-ctg50} or security games
\cite{BasNitGat-aaai16}.
This paper answers this question in the setting of
2-player zero-sum partially observable stochastic games (zs-POSGs),
\ie, imperfect information games with simultaneous moves, perfect
recall, discounted rewards and a possibly infinite time horizon.

As general POSGs and Dec-POMDPs, infinite-horizon zs-POSGs are
undecidable, and their finite-horizon approximations are in NEXP
\cite{MadHanCon-aaai99,Bernstein02}.
As further discussed in Section~\ref{sec|relatedWork}, solution
techniques for finite-horizon POSGs, or other imperfect information
games that can be formulated as extensive-form games (EFGs), typically
solve an equivalent normal-form game \cite{ShoLey09} or use a
dedicated regret-minimization mechanism
\cite{ZinJohBowPic-nips07,BroSan-science18}.
They thus do not rely on Bellman's optimality principle, except
(i) a dynamic programming approach that only constructs sets of
non-dominated solutions \cite{HanZil-aaai04},
(ii) in collaborative problems (Decentralized POMDPs), adopting the
viewpoint of a (blind) central planner
\cite{SzeChaZil-uai05,DibAmaBufCha-jair16},
and
(iii) for (mostly 2-player zero-sum) settings with observability
assumptions such that one can reason on player beliefs
\cite{GhoMcDSin-jota04,ChaDoy-tcl14,BasSte-jco15,HorBosPec-aaai17,ColKoc-jet01,HorBos-aaai19}.
Here, we do not make any assumption beyond the game being 2-player
zero-sum, in particular regarding observability of the state and
actions.

As for a number of Dec-POMDP solvers, our approach adopts the viewpoint
not of a player, but of a central (offline) planner that prescribes
individual strategies to the players \cite{SzeChaZil-uai05},
which allows turning a zs-POSG into a non-observable game
for which Bellman's optimality principle applies.
This is achieved in Section~\ref{sec|oSG}
(after background Section~\ref{sec|background})
while reasoning not on a player's belief over the game state (as
feasible in POMDPs or some particular games), but on the central
planner's (blind) belief, a statistic called occupancy state and that
we prove to be sufficient for optimal planning, as
Dibangoye et al. did for Dec-POMDPs  \cite{DibAmaBufCha-jair16}.
In Section~\ref{sec|approxV}, our Bellman/Shapley operator is proved
to induce an optimal game value function that is Lipschitz-continuous
in occupancy space, which leads to
deriving value function approximators, including upper- and
lower-bounding ones, and
discussing their initialization.
Finally, Section~\ref{sec|HSVI} describes a variant of HSVI for
zs-POSGs, and demonstrates its finite-time convergence to an
$\epsilon$-optimal solution despite the continuous (occupancy) state
and action spaces.

\section{Related Work}
\label{sec|relatedWork}

Infinite horizon POSGs are undecidable \cite{MadHanCon-aaai99}, which
justifies searching for near-optimal solutions, \eg, through finite
horizon solutions, as we will do.
There is little work on solving POSGs, in particular through
exploiting Bellman's optimality principle.
One exception is Hansen and Zilberstein's work on finite horizon POSGs \cite{HanZil-aaai04} ,
where dynamic programming (DP) incrementally constructs non-dominated
policy trees for each player, which allows then deriving a solver for
common-payoff POSGs, \ie, decentralized partially observable Markov
decision processes (Dec-POMDPs).
Here, Bellman's OP thus serves as a pre-processing phase, while we
aim at employing it in the core of algorithms.

\paragraph{Dec-POMDPs}

Bellman's OP appears as the core component of a Dec-POMDP solver when
Szer et al. \cite{SzeChaZil-uai05} adopt a planner-centric viewpoint whereby the
planner aims at providing the players with their private policies
without knowing which action-observation histories they have
experienced.
The planner's information state at $t$ thus contains the initial
belief and the joint policy up to $t$.
This leads to turning a Dec-POMDP into an information-state MDP, and
obtaining a deterministic shortest path problem that can be solved
using an A* search called MAA* (multi-agent A*).

Then, another important step is when Dibangoye et al. \cite{DibAmaBufCha-jair16} show
that
(i) the {\em occupancy state}, a statistic used to compute expected
rewards in MAA*, is in fact sufficient for planning, and
(ii) the optimal value function is piecewise linear and convex (PWLC)
in occupancy space, which allows adapting point-based POMDP solvers
using approximators of $V^*$.

\paragraph{Subclasses of POSGs}

Recent works addressed particular cases of discounted partially observable stochastic games (POSGs), 2-player and zero-sum
if not specified otherwise, exploiting the structure of the problem to
turn it into an equivalent problem for which Bellman's principle
applies.
Ghosh et al. \cite{GhoMcDSin-jota04} considered POSGs with public actions and
shared observations,
which can be turned into stochastic games defined over the common
belief space, similarly to POMDPs turned into belief MDPs.
Chatterjee and Doyen \cite{ChaDoy-tcl14}, Basu and Stettner \cite{BasSte-jco15}, and Horák et al. \cite{HorBosPec-aaai17} considered One-Sided
POSGs, \ie, scenarios where (player) $2$ (w.l.o.g.) only partially
observes the system state, and \player $1$ has access to the system
state, plus the action and observation of \player $2$.
Cole and Kocherlakota \cite{ColKoc-jet01} considered ($n$-player) POSGs with independent
private states, partially shared observability, and \player $i$'s
utility function depending on his private state and on the shared
observation.
Horák and Bošanský \cite{HorBos-aaai19} considered zs-POSGs with independent private
states and public observations, \ie, scenarios where (i) each player
$i$ has a private state $s_i$ he fully observes, and (ii) both players
receive the same public observations of each player's private
state.
{Any player's belief over the other player's private state is thus
  common knowledge.}

Focusing on the work of Horák et al. \cite{HorBosPec-aaai17,HorBos-aaai19},
in both cases convexity or concavity properties of the optimal value
function are obtained, which allow deriving upper- and lower-bounding
approximators.
These approximators are then employed in HSVI-based algorithms.
Yet, moving from MDPs and POMDPs (as in Smith's
work) to these settings induces a tree of possible futures with an
infinite branching factor, which requires changes to the algorithm,
and thus to the theoretical analysis of the finite-time convergence.
As we shall see, the present work adopts similar changes.

Wiggers et al. \cite{WigOliRoi-ecai16} prove that, using appropriate
representations, the value function associated to a zs-POSG is convex
for (maximizing) player $1$ and concave for (minimizing) player $2$.
Yet, this did not allow deriving a solver based on approximating the
value function.
Here, we exploit no convexity or concavity property of the optimal
value function, as they may not hold, but its Lipschitz continuity.

\paragraph{Imperfect Information Games}

Finite horizon (general-sum) POSGs can be written as extensive-form
games with imperfect information and perfect recall (EFGs, often
referred to as {\em imperfect information games}) \cite{OliVla-tr06},
which makes solution techniques for EFGs relevant even for
infinite-horizon POSGs.
A first approach to solving EFGs is to turn them into a normal-form
game before looking for a Nash equilibrium, thus ignoring the temporal
aspect of the problem \cite{ShoLey09} and inducing a combinatorial
explosion.
For (2-player) zs-EFGs, this leads to solving two linear programs (one
for each player).
Koller and Megiddo \cite{KolMeg-geb92} propose a different linear programming approach
for zs-EFGs that exploits the temporal aspect through the choice of
decision variables, but still does not apply Bellman's OP (see also
\cite{Stengel-geb96,KolMegSte-geb96}).

More recently, Counterfactual Regret minimization (CFR)
\cite{ZinJohBowPic-nips07} has been introduced, allowing to solve
large imperfect-information games with bounded regret such as heads-up
no limit hold’em poker, now winning against top human players
\cite{BroSan-science18}.
While some CFR-based algorithms use heuristic-search techniques, thus
somehow exploit the sequentiality of the game, they do not rely on
Bellman's OP either.

\section{Background}
\label{sec|background}

For the sake of clarity, the concepts and results of the EFG
literature used in this work will be recast in the POSG setting.
We shall employ the terminology of pure/mixed/behavioral strategies
and strategy profiles---more convenient in our non-collaborative
setting---instead of deterministic or stochastic policies (private or
joint ones)---common in the collaborative setting of Dec-POMDPs.

A (2-player) zero-sum partially observable stochastic game (zs-POSG)
is defined by a tuple
$\langle \cS, \cA^1, \cA^2, \cZ^1, \cZ^2, P, r, H, \gamma, b_0 \rangle$, where
\begin{itemize}
\item $\cS$ is a finite set of states;
\item $\cA^i$ is (player) $i$'s finite set of actions;
\item $\cZ^i$ is \player $i$'s finite set of observations;
\item $\PP{s}{a^1,a^2}{s'}{z^1,z^2}$ is the probability to transition to
  state $s'$ and receive observations $z^1$ and $z^2$ when actions
  $a^1$ and $a^2$ are performed in state $s$;
\item $r(s,a^1,a^2)$ is a (scalar) reward function (bounded by $[r_{\min},r_{\max}]$, \ie,
  $r_{\min} \eqdef \min_{s,a^1,a^2} r(s,a^1,a^2)$ and
  $r_{\max}\eqdef \max_{s,a^1,a^2} r(s,a^1,a^2)$);
\item $H \in \mathbb{N} \cup \{\infty\}$ is a temporal horizon;
\item $\gamma\in [0,1]$ is a discount factor ($H=\infty$ implies $\gamma<1$); and
\item $b_0$ is the (public/common) initial belief state.
\end{itemize}
\Player $1$ would like to maximize the expected return, defined as the
discounted sum of future rewards, while \player $2$ would like to
minimize it, what we formalize next.

From the Dec-POMDP, POSG and EFG literature, we use the
following concepts and definitions, where $i \in \{1,2\}$:
\begin{description}
\item[\Player $-i$] is \player $i$'s opponent.
\item[$\theta^i_\depth = (a^i_1, z^i_1, \dots , a^i_\depth, z^i_\depth) $] is a
  length-$\depth$ {\em action-observation history} for \player $i$.
The set of histories is
  $\Theta^i = \Theta^i_0 \cup \Theta^i_1 \cup \Theta^i_2 \cup
  \dots$, with one subset per time step.
\item[$\vth_\depth=(\theta^1_\depth,\theta^2_\depth)$] is a {\em joint
    history} at time $\depth$.
The set of joint histories is
  $\vTh = \vTh_0 \cup \vTh_1 \cup \vTh_2 \cup
  \dots$, with one subset per time step.
\item[{[$\occ_\depth$]}] An {\em occupancy state} $\occ_\depth$ at
  time $\depth$ is a probability distribution over
  state--joint-history pairs $(s,\vth_\depth)$.
($\occ_0$ is completely specified by $b_0$.)
The set of occupancy states is
  $\Occ = \Occ_0 \cup \Occ_1 \cup \Occ_2 \cup \dots $, with one subset
  per time step.
Note that this notion applies to POSGs despite the use of stochastic
  actions.
\item[{[$\pi^i_{0:\depth}$]}] A {\em pure strategy} for \player $i$ is
  a mapping $\pi^i_{0:\depth}$ from private histories in $\Theta^i_t$
  ($\forall t \in \{0\twodots \depth\}$) to {\bf single} private
  actions. By default, $\pi^i\eqdef \pi^i_{0:H-1}$.
\item[$\vpi_{0:\depth}$]$= \langle \pi^1_{0:\depth}, \pi^2_{0:\depth}
  \rangle$ is a {\em pure strategy profile}.
\item[{[$\mu^i_{0:\depth}$]}] A {\em mixed strategy}
  $\mu^i_{0:\depth}$ for \player $i$ is a probability distribution over
  pure strategies.
It is used by first sampling one of the pure strategies (at $t=0$),
  and then executing that strategy until $t=\depth$.
\item[$\vmu_{0:\depth}$]
  $= \langle \mu^1_{0:\depth}, \mu^2_{0:\depth} \rangle$ is a {\em
    mixed strategy profile}.
\item[{[$\beta^i_\depth$]}] A {\em (behavioral) decision rule} at time
  $\depth$ for \player $i$ is a mapping $\beta^i_\depth$ from private
  histories in $\Theta^i_\depth$ to {\bf distributions} over private
  actions.
For convenience, we will note $\beta^i_\depth(\theta^i_\depth,a^i)$
  the probability to pick action $a^i$ when facing history
  $\theta^i_\depth$.
\item[$\vbeta_\depth$]
  $= \langle \beta^1_\depth, \beta^2_\depth \rangle$ is a {\em
    decision rule profile} ($\in \cB_\depth$, and noting
  $\cB = \cB_0\cup\cB_1\cup \dots$).
\item[$\beta^i_{\depth:\depth'}$]
  $= (\beta^i_\depth, \dots, \beta^i_{\depth'})$ is a {\em behavioral
    strategy} for \player $i$ from time step $\depth$ to $\depth'$
  (included).
By default, $\beta^i \eqdef \beta^i_{0:H-1}$.
\item[$\vbeta_{\depth:\depth'}$]
  $= \langle \beta^1_{\depth:\depth'}, \beta^2_{\depth:\depth'}
  \rangle$ is a {\em behavioral strategy profile}.
\item[{[$V_0(\occ_0,\vbeta) $]}] The {\em value} of a behavioral strategy profile $\vbeta$ in
  occupancy state $\occ_0$ (from time step $0$ on) is:
  \begin{align*}
    V_0(\occ_0,\vbeta) 
    & = E[\sum_{t=0}^\infty \gamma^t R_t | \Occ_0 = \occ_0, \vbeta],
  \end{align*}
  where $R_t$ is the random variable associated to the instant reward
  at time step $t$. 
[Note: This definition extends naturally to pure and mixed strategy
  profiles.]
\end{description}

The primary objective is here to find a Nash equilibrium strategy
(NES), \ie, a mixed strategy profile
$\vmu^*=\langle\mu^{1*},\mu^{2*}\rangle$ such that no player has an
incentive to deviate, which can be written:
\begin{align*}
  \forall \mu^1, V_0(\occ_0,{\mu^{1*},\mu^{2*}}) & \geq V_0(\occ_0,{\mu^{1},\mu^{2*}}), \\
  \forall \mu^2, V_0(\occ_0,{\mu^{1*},\mu^{2*}}) & \leq V_0(\occ_0,{\mu^{1*},\mu^{2}}).
\end{align*}
In such a 2-player zero-sum game, all NESs have the same
Nash-equilibrium value (NEV)
$V^*_0(\occ_0) \eqdef V_0(\occ_0,{\mu^{1*},\mu^{2*}})$.

Finite horizon POSGs being equivalent to EFGs with imperfect
information and perfect recall, the following key result for EFGs
applies to (finite $H$) POSGs:
\begin{theorem}{\cite{Kuhn-ctg53,FudTir-gt91}}
  In a game of perfect recall, mixed and behavior strategies are
  equivalent.
(More precisely: Every mixed strategy is equivalent to the unique
  behavior strategy it generates, and each behavior strategy is
  equivalent to every mixed strategy that generates it.)
\end{theorem}

\section{Solving POSGs as Occupancy MGs}
\label{sec|oSG}

In this section, unless stated otherwise, we assume finite horizons
and exact solutions (no $\epsilon$ error).

Here, we show
(i) how a zs-POSG can be reformulated as a different zero-sum Markov
game, and
(ii) that Bellman's optimality principle applies in this game.

\subsection{From zs-POSGs to zs-OMGs}

To solve a zs-POSG, we take the viewpoint of a central planner that
searches {\em offline} for the best behavioral strategy profile before
providing it to the players.
This contrasts with Dec-POMDPs where deterministic strategy profiles
suffice,
and means exploring a (bounded) continuous space rather than a
(finite) discrete one as for Dec-POMDPs.
Such a planner grows a partial strategy $\vbeta_{0:\depth-1}$ by
appending a decision rule profile $\vbeta_{\depth}$.
The controlled process induced in occupancy space, where actions are
decision rule profiles, is both deterministic and Markovian (see
formal details about the dynamics below): applying $\vbeta_\depth$ in
$\occ_{\depth}$ (\ie, appending it to $\vbeta_{0:\depth}$) leads to a
unique $\occ_{\depth+1}$.
Also, the expected reward at time $\depth$ is linear in occupancy
space (more precisely in the corresponding distribution over states).
All this allows reasoning not on partial behavioral strategy profiles,
but on occupancy states.
The central planner will thus
(i) infer occupancy states seen as ``beliefs'' over the possible
situations (``situation'' here meaning the current state $s$ and the
players' joint action-observation history $\vth_\depth$) which may
have been reached, although without knowing what actually happened,
and
(ii) map each occupancy state to a decision rule profile telling the
players how to act depending on their actual action-observation
histories.\footnote{In contrast, in a POMDP, the belief state depends
  on the agent's action-observation history, and is mapped to a single
  action.}
Each zs-POSG is thus turned into an equivalent game, called a {\em
  zero-sum occupancy Markov game} (zs-OMG)\footnote{We use (i)
  ``Markov game'' instead of ``stochastic game'' because the dynamics
  are not stochastic, and (ii) ``partially observable stochastic
  game'' to stick with the literature.}
formally defined by the tuple
$\langle \Occ, \cB, \nxt, r, H, \gamma, b_0 \rangle$, where:
\begin{itemize}
\item $\Occ$ is the set of occupancy states induced by the zs-POSG;
\item $\cB$ is the set of decision rule profiles of the zs-POSG;
\item $\nxt$ is a deterministic transition function that maps each
  pair $(\occ_\depth,\vbeta_\depth)$ to the (only) possible next
  occupancy state $\occ_{\depth+1}$; formally (see Lemma~\extref{lem:nextOcc} in App.~\extref{app:fromTo}),
  $\forall s', \theta^1_{\depth}, a^1, z^1, \theta^2_{\depth}, a^2,
  z^2$,
  \begin{align*}
\ifextended{\hspace*{\dimexpr-\leftmargini}}{}
    & \nxt(\occ_\depth,\vbeta_\depth)(s',(\theta^1_{\depth},a^1,z^1),(\theta^2_{\depth},a^2,z^2)) \\
    \ifextended{\hspace*{\dimexpr-\leftmargini}}{}
    & \eqdef Pr(s',(\theta^1_{\depth},a^1,z^1),(\theta^2_{\depth},a^2,z^2)) \\
    \ifextended{\hspace*{\dimexpr-\leftmargini}}{}
    & 
= \beta^1_{\depth}(\theta^1_{\depth},a^1) \beta^2_{\depth}(\theta^2_{\depth},a^2)  \sum_{s} P^{z^1,z^2}_{a^1,a^2}(s'|s) 
      \occ_{\depth} (s,\theta^1_{\depth},\theta^2_{\depth}));
\end{align*}
\item $r$ is a reward function naturally induced from the zs-POSG as
  the expected reward for the current occupancy state and decision
  rule profile:
  \begin{align*}
\ifextended{\hspace*{\dimexpr-\leftmargini}}{}
    & r(\occ_\depth,\vbeta_\depth) 
    \eqdef E[r(S,A^1,A^2) | \occ_\depth, \beta^1_\depth, \beta^2_\depth ] \\
    \ifextended{\hspace*{\dimexpr-\leftmargini}}{}
    & = \sum_{s,\vth_\depth} \occ_\depth(s,\theta^1_\depth,\theta^2_\depth) 
      \sum_{a^1,a^2} \beta^1_\depth(\theta^1,a^1) \beta^2_\depth(\theta^2,a^2)
      r(s,a^1,a^2);
  \end{align*}
  we use the same notation $r$ for zs-POSGs as the context shall
  indicate which one is discussed;
\item $H$, $\gamma$, and $b_0$ are as in the zs-POSG.
\end{itemize}

Note first that, for convenience, we directly consider behavioral
decision rules, which correspond to mixed strategies.
Of course, at $\depth$, \player $i$'s possible actions should be
decision rules defined over histories that have non-zero probability
in current $\occ_\depth$.
The dynamics being deterministic and the actions public, both players
of that new game (also denoted $1$ and $2$ while these are different
players) know the next state after each transition.
But this is no standard zs Markov game also since
(i) the mixture of two actions is equivalent to another action already
in the (continuous) action space at hand, and
(ii) at each time step, the state (occupancy) space $\Occ_\depth$ is
continuous.

We shall study the {\em subgames} of a zs-OMG,
\ie, situations where some occupancy state $\occ_\depth$ has somehow
been reached at time step $\depth$, and the central solver is looking
for rational strategies ($\beta^1_{\depth:H-1}$ and
$\beta^2_{\depth:H-1}$) to provide to the players.
$\occ_\depth$ tells which action-observation histories each player
could be facing with non-zero probability, and thus which are relevant
for planning.
We can then extend the definition of value function from time step $0$
only to any time step $\depth$ as follows (using behavioral
strategies):
\begin{align*}
  & V_\depth(\occ_\depth,\beta^1_{\depth:H-1},\beta^2_{\depth:H-1}) \\
  & = E[\sum_{t=\depth}^\infty \gamma^{t-\depth} R_t | \Occ_\depth = \occ_\depth, \beta^1_{\depth:H-1}, \beta^2_{\depth:H-1}].
\end{align*}

\subsubsection{What We Are Looking For}

For any $\depth$ and $\occ_\depth$, let us define
\begin{itemize}
\item $\vbeta^*_{\depth:H-1}(\occ_\depth)$ a NE profile for the
  subgame at $\occ_\depth$, and
\item $V^*_\depth(\occ_\depth)$ the NE value of the subgame at any
  $\occ_\depth$.
\end{itemize}
Let also $first(\cdot)$ and $rest(\cdot)$ be two functions that map a
behavioral strategy (profile) defined over $t:t'$ (where $t<t'$)
respectively to
(i) its decision rule at $t$, and
(ii) its restriction over $t+1:t'$.
If Bellman's optimality principle holds, then we expect
$rest(\vbeta^*_{\depth:H-1}(\occ_\depth))$ to be an optimal solution
of the subgame at
$\nxt(\occ_\depth,first(\vbeta^*_{\depth:H-1}(\occ_\depth))$.
Assuming that we have the optimal solution of any subgame at
$\depth+1$, \ie, $\vbeta^*_{\depth+1:H-1}(\cdot)$ is known, and thus
that $V^*_{\depth+1}(\cdot)$ is known, then one can solve the subgame
at $\occ_\depth$ by (i) solving the {\em local} game:
\begin{align*}
  Q^*_{\depth}(\occ_\depth, \vbeta_\depth)
  & \eqdef r(\occ_\depth, \vbeta_\depth) + \gamma V^*_{\depth+1}( \nxt(\occ_\depth,\vbeta_\depth) ),
\end{align*}
and (ii) then concatenating an optimal solution
$\vbeta^*_{\depth}(\occ_\depth)$ with an optimal solution of the
induced subgame at $\nxt(\occ_\depth,\vbeta^*_{\depth}(\occ_\depth))$,
\ie,
$\vbeta^*_{\depth+1:H-1}(\nxt(\occ_\depth,\vbeta^*_{\depth}(\occ_\depth)))$.

This approach would only require an algorithm manipulating behavioral
strategies.
Yet, to demonstrate that Bellman's optimality principle holds and that
this approach is valid, we will need to reason on mixed strategies.
The following section thus presents observations and preliminary
results about behavioral and mixed strategies.

\subsection{Compatibility of Strategies}

\subsubsection{Comments About Behavioral Strategies}

First, note that the probability of some action-observation history
$\theta^i_t$ (possibly completed with an action) under $\occ_\depth$
is given by
\begin{align}
 P(\theta^i_t|\occ_\depth) 
  & = \sum_{\substack{ s,\ \theta^{-i}_{\tau-1} \\ \theta^i_{\tau-1} \in \Theta^i_{\depth-1}(\theta^i_t)}} o_\tau (s,\theta^1_{\tau-1} , \theta^2_{\tau-1} ),
\end{align}
where $\Theta^i_{\depth-1}(\theta^i_t)$ is the set of $i$'s histories of
length $\depth-1$ that are prefixed with $\theta^i_t$.
With this, let us observe that a reachable occupancy state
$\occ_\depth$ may have been generated by multiple (prefix) behavioral
strategy profiles $\vbeta_{0:\depth-1}$, as long as, for any $a^i$ and
any $\theta^i_{0:t}$ ($t<\depth$) that has non-zero probability in
$\occ_\depth$,
\begin{align}
  \beta^i_{0:\tau-1}(\theta^i_{0:t},a^i) 
  & = P(a^i|\occ_\depth,\theta^i_t) 
    = \frac{P(\theta^i_t,a^i|\occ_\depth)}{P(\theta^i_t|\occ_\depth)}.
\end{align}
Such a behavioral strategy profile is said to be {\em compatible} with
$\occ_\depth$.
We will denote $\cB^i_{0:t|\occ_\depth\rangle}$ the set of behavioral
strategies of $i$ compatible with $\occ_\depth$, with $t$ often being
$\depth-1$ or $H-1$ (for conciseness, $H-1$ may be omitted).

Two important comments on compatible behavioral strategies are the
following:
\begin{itemize}
\item Let $\hat\vbeta_{0:\depth-1}$ be the partial behavioral strategy
  profile that actually led to $\occ_\depth$.
Then
(i) both $\hat\beta^i_{0:\depth-1}$ and
  $\hat\beta^{-i}_{0:\depth-1}$ influence which histories of $i$ have
  non-zero probability, but
(ii) only $\hat\beta^i_{0:\depth-1}$ influences the action-selection
  probabilities.
\item Any compatible $\beta^i_{0:\depth-1}$ can be extended in time in
  any manner without altering $\occ_\depth$.
In particular, any optimal solution of the subgame at $\occ_\depth$
  can be concatenated with any prefix behavioral strategy profile that
  may have led to $\occ_\depth$.
And, as a matter of fact, one will typically solve a subgame at
  $\occ_\depth$ without accounting for the prefix behavioral strategy
  profile that may have induced $\occ_\depth$.
\end{itemize}

In the same vein, let us consider the set
$\cB^i_{|\occ_\depth,\beta^j_\depth\rangle}$ of behavioral strategies
of $i$ that are compatible with $\occ_\depth$ and exploit the
knowledge of some additional behavioral decision rule
$\beta^j_\depth$.
In this case, 
\begin{itemize}
\item if $i=j$, then the strategies should satisfy $\beta^j_\depth$,
  \ie, induce the same probabilistic action choices (at least for
  reachable histories); and
\item if $i\neq j$, then some action-observation histories
  $\theta^i_\depth$ may become impossible due to $\beta^j_\depth$, so
  that $i$ can ignore them in practice.
\end{itemize}

To go a step further, we can note the following two points:
\begin{itemize}
\item the set $\cB^i_{|\occ_\depth,\vbeta_\depth\rangle}$ of
  behavioral strategies of $i$ that are compatible with $\occ_\depth$
  and exploit the knowledge of some additional behavioral decision
  rule profile $\vbeta_\depth$ defines a subset of
  $\cB^i_{|\nxt(\occ_\depth,\vbeta_\depth)\rangle}$,
because knowing $\occ_\depth$ and $\vbeta_\depth$ induces the same
  next occupancy state, but with the added constraint of knowing which
  exact behavioral decision rule $\beta^i_\depth$ has been followed;
yet, only the suffixes of the behavioral strategies are relevant for
  the value of the sub-game at $\nxt(\occ_\depth,\vbeta_\depth)$, so
  that both sets can be employed interchangeably;
\item following the previous observations, the set
  $\cB^i_{|\occ_\depth,\vbeta_\depth\rangle}$ can even be replaced by
  $\cB^i_{|\occ_\depth,\beta^i_\depth\rangle}$ when it comes to
  reasoning on the optimal value of a sub-game.
\end{itemize}

\subsubsection{Moving to Mixed Strategies}

Mixed strategies are usually defined from time step $0$ on,
and used by
(1) sampling at $t=0$ a pure strategy from the distribution it
specifies, then
(2) executing that pure strategy from then on.
While this may not seem appropriate at first sight, we will still use
mixed strategies defined from time step $0$ to reason on a subgame at
$\occ_\depth$, but considering the subset of strategies {\em
  compatible} with $\occ_\depth$, \ie, that induce that occupancy
state at $\depth$.

As can be noted, a mixed strategy $\mu^i$ is compatible with
$\occ_\depth$ if and only if its equivalent behavioral strategy
$\beta(\mu^i)$ is compatible with $\occ_\depth$.
Let $M^i_{0:H-1|\occ_\depth\rangle}$ be the set of mixed strategies
$\mu^i_{0:H-1|\occ_\depth\rangle}$ compatible with $\occ_\depth$
(often noted respectively $M^i_{|\occ_\depth\rangle}$ and
$\mu^i_{|\occ_\depth\rangle}$).
This allows reasoning interchangeably with behavioral or mixed
strategies of length $H-1$ as long as they are compatible with
$\occ_\depth$.

To see the benefit of using mixed strategies, let us consider the
subgame at $\occ_\depth$ in the space of compatible mixed strategy
profiles $M^i_{0:H-1|\occ_\depth\rangle}$.
The following lemma gives us a first observation on that space.

\begin{restatable}[Proof in App.~\extref{proofLemEquivalence}]{lemma}{lemConvexity}
  \label{lem:convexity}
  \IfAppendix{{\em (originally stated on
    page~\pageref{lem:convexity})}}{}
$M^i_{0:H-1|\occ_\depth\rangle}$ is convex.
\end{restatable}

In addition,
$V_\depth(\occ_\depth,\mu^1_{|\occ_\depth\rangle},\mu^2_{|\occ_\depth\rangle})$
is linear in both strategy spaces (as $V_0$ is at the initial time
step).
Because of this bi-linearity and of the convexity of both mixed
strategy spaces, we are facing a normal-form game and can apply von
Neumann's Minimax theorem, \ie, find solutions to the subgame by computing the security levels
for each player:
\begin{align*}
 & \max_{\mu^1_{|\occ_\depth\rangle} \in M^1_{|\occ_\depth\rangle}}
 \min_{\mu^2_{|\occ_\depth\rangle} \in M^2_{|\occ_\depth\rangle}}
 V_\depth(\occ_\depth,\mu^1_{|\occ_\depth\rangle},\mu^2_{|\occ_\depth\rangle}), \text{ and} \\
 & \min_{\mu^2_{|\occ_\depth\rangle} \in M^2_{|\occ_\depth\rangle}}
 \max_{\mu^1_{|\occ_\depth\rangle} \in M^1_{|\occ_\depth\rangle}}
 V_\depth(\occ_\depth,\mu^1_{|\occ_\depth\rangle},\mu^2_{|\occ_\depth\rangle}).
\end{align*}
But note that these formulas will only serve theoretical purposes.

As can be easily demonstrated (cf. Lemma~\extref{lem:nestedNEs} in
App.~\extref{app:subNEs}), any Nash equilibrium solution of our original game
$V_0(\occ_0,\cdot,\cdot)$ induces a Nash equilibrium in any of its
successive reachable subgames.
But this does not tell whether Bellman's optimality principle applies,
what we discuss next.

\subsection{Bellman's Optimality Principle}

For any $\depth$ and $\occ_\depth$, let us define
(i) $\vbeta^*_{\depth:H-1}(\occ_\depth)$ a NE profile for the subgame at
$\occ_\depth$,
(ii) $V^*_\depth(\occ_\depth)$ the NE value of the subgame at
any $\occ_\depth$, and
(iii) the {\em local} subgame at $\occ_\depth$
\begin{align*}
  Q^*_{\depth}(\occ_\depth, \vbeta_\depth)
  & \eqdef r(\occ_\depth, \vbeta_\depth) + \gamma V^*_{\depth+1}( \nxt(\occ_\depth,\vbeta_\depth) ). \end{align*}
Then, given Nash equilibrium solutions for any $\occ_{\depth+1}$, the
applicability of Bellman's optimality principle shall be proved if a
Nash equilibrium of $V_\depth(\occ_\depth,\vbeta_{\depth:H-1})$ can be
found by
(i) solving the local subgame
$Q^*_{\depth}(\occ_\depth, \vbeta_\depth)$ to get a decision rule
profile $\vbeta^*_\depth$ and
(ii) appending it to
$\vbeta^*_{\depth+1:H-1}(\nxt(\occ_\depth, \vbeta^*_\depth))$.

\paragraph{An Abnormal-Form Game?}

A first question is whether this game
$Q^*_{\depth}(\occ_\depth, \vbeta_\depth)$ is in fact a normal-form
game, \ie, whether it could be defined by a payoff matrix over pure
decision rules, payoffs for behavioral decision rules being
obtained through linear mixtures.

$V_\depth(\occ_\depth,\cdot,\cdot)$ is linear in each player's
decision rule space at each time step (\ie, in $\beta^i_{\depth'}$ for
any $i$ and $\depth'\in\{\depth \twodots H-1\}$), but multilinear in
each player's behavioral strategy space (see Lemma~\extref{cor|V|lin|dr} App.~\extref{appLinearity}),
which suggests that $Q^*_{\depth}(\occ_\depth, \cdot, \cdot)$ is not
(bi)linear in the space of decision rules at $\depth$, and thus
possibly not concave-convex.
As a consequence, we are possibly facing an {\bf ab}normal-form game
and cannot use von Neumann's Minimax theorem.

\paragraph{Properties of the Maximin and Minimax Values}

Rather than digging the concavity-convexity property further, we now
show that computing the maximin and minimax values of
$Q^*_\depth(\occ_\depth, \vbeta_\depth)$ induces finding a NE of
$V_\depth(\occ_\depth, \vbeta_{\depth:H-1})$ given NEs for any
$\occ_{\depth+1}$.

\begin{restatable}[Proof in App.~\extref{proofLemAbnormalMaximinimax}]{theorem}{lemAbnormalMaximinimax}
  \label{lem:abnormalMaximinimax}
  \IfAppendix{{\em (originally stated on
    page~\pageref{lem:abnormalMaximinimax})}}{}
In the 2p zs abnormal-form game
  $Q^*_\depth(\occ_\depth,\vbeta_\depth)$, the maximin and minimax
  values are both equal to $V^*_\depth(\occ_\depth)$---\ie, as
  previously defined, the NEV for game
  $V_\depth(\occ_\depth,\vbeta_{\depth:H-1})$---and correspond to a
  NES.
\end{restatable}
  
\begin{proof}{(sketch)}
  The proof relies on first developing the maximin of
  $Q^*_\depth(\occ_\depth,\beta^1_\depth,\beta^2_\depth)$, then using
(i) the equivalence of maximin and minimax for mixed strategies
  (as von Neumann's minimax theorem applies), and
(ii) the equivalence of mixed and behavioral strategies.
\end{proof}

\begin{restatable}[Proof in App.~\extref{proofCorAbnormalMaximinimax}]{theorem}{corAbnormalMaximinimax}
  \label{cor:abnormalMaximinimax}
  \IfAppendix{{\em (originally stated on
    page~\pageref{cor:abnormalMaximinimax})}}{}
  As in 2p zs normal-form games, game
  $Q^*_\depth(\occ_\depth,\vbeta_\depth)$ has at least one NES; all
  its NESs are all value-equivalent; and solving for maximin and
  minimax values allows finding one NES.
\end{restatable}

\paragraph{Maximin and Minimax Computation}
\label{sec|maximinimax}

The last results tell us that we can exploit knowledge of the optimal
value function at $\depth+1$ (for all $\occ_{\depth+1}$) to find
optimal decision rules at $\depth$ for any given $\occ_\depth$ by computing the maximin and minimax values of the local
(abnormal-form) game at hand.
Yet, we cannot use an LP as for normal-form games.
To find an appropriate solution method, let us now look at properties
of this game, noting that we lack any convexity/concavity property,
and starting with a preliminary result.

\begin{restatable}[Proof in App.~\extref{proofLemOccLin}]{lemma}{lemOccLin}
  \label{lem|occ|lin}
  \IfAppendix{{\em (originally stated on
    page~\pageref{lem|occ|lin})}}{}
At depth $\depth$, $\nxt(\occ_\depth,\vbeta_\depth)$ is linear in
  $\beta^1_\depth$, $\beta^2_\depth$, and $\occ_\depth$, where
  $\vbeta_\depth=\langle \beta^1_\depth, \beta^2_\depth\rangle$.
It is more precisely $1$-Lipschitz-continuous in $\occ_\depth$ (in
  $1$-norm), \ie, for any $\occ_\depth$, $\occ'_\depth$:
  \begin{align*}
    \norm{\nxt(\occ'_\depth,\vbeta_\depth) - \nxt(\occ_\depth,\vbeta_\depth)}_1
    & \leq 1 \cdot \norm{\occ'_\depth - \occ_\depth}_1.
  \end{align*}
\end{restatable}

The Lipschitz continuity (LC) property would also hold in $2$-norm or
$\infty$-norm, due to the equivalence between norms, but with
different constants.

\begin{restatable}[Proof in App.~\extref{proofLemQLC}]{lemma}{lemQLC}
  \label{lem|Q|LC}
  \IfAppendix{{\em (originally stated on
      page~\pageref{lem|Q|LC})}}{}
For any $\depth$ and $\occ_\depth$,
  $Q^*_\depth(\occ_\depth,\beta^1_\depth,\beta^2_\depth)$ is Lipschitz
  continuous in both $\beta^1_\depth$ and $\beta^2_\depth$.
\end{restatable}

The payoff function of our game is thus LC in each private
decision-rule space, which suggests using error-bounded global
optimization techniques, as Munos's DOO (Deterministic
Optimistic Optimization) \cite{Munos-ftml14}.
Here, searching for a maximin (resp. minimax) value suggests using two
nested optimization processes: an ``outer'' one for the $\max$ (resp. $\min$)
operator, and an ``inner'' one for the $\min$ (resp. $\max$).
To ensure being within $\epsilon$ of the maximin value, each process
could, for example, use an $\frac{\epsilon}{2}$ tolerance
threshold. Yet, in such a nested optimization process, the inner process may
stop, at each call, before reaching $\frac{\epsilon}{2}$-optimality if
it leads the outer process to explore a different point.

Due to the continuous state space of zs-OMGs, $V^*$ cannot be computed
exactly.
We shall now see how to approximate it, before exploiting the
resulting approximators in a specific version of HSVI in
Sec.~\ref{sec|HSVI}.

\section{Properties of $V^*$}
\label{sec|approxV}

In this section, we again assume finite horizon problems (unless
stated otherwise).
The main objective here is to propose upper- and lower-bounding
approximators that exploit $V^*$'s Lipschitz continuity (rather than
PWLC) property, as Fehr et al. \cite{FehBufThoDib-nips18} did in the setting of
(single agent) information-oriented control, but here with simpler
derivations.

\subsection{Finite-Horizon Lipschitz Continuity of $V^*$}
\label{sec|LC|V}

The following lemma proves that the expected instant reward at any
$\depth$ is linear in $\occ_\depth$, and thus so is the expected value
of a finite-horizon strategy profile from $\depth$ onwards (trivial
proof by induction). 

\begin{restatable}[Proof in App.~\extref{proofLemVLinOcc}]{lemma}{lemVLinOcc}
  \label{lem|V|lin|occ}
  \IfAppendix{{\em (originally stated on
    page~\pageref{lem|V|lin|occ})}}{}
At depth $\depth$,
  $V_\depth(\occ_\depth,\vbeta_{\depth:H-1})$ is linear w.r.t.  $\occ_\depth$.
\end{restatable}

\begin{restatable}[Proof in App.~\extref{proofCorVLCOcc}]{corollary}{corVLCOcc}
  \label{cor|V|LC|occ}
  \IfAppendix{{\em (originally stated on
    page~\pageref{cor|V|LC|occ})}}{}
$V^*_\depth(\occ_\depth)$ is Lipschitz continuous in $\occ_\depth$
  at any depth $\depth \in \{0 \twodots H-1\}$.
\end{restatable}

\paragraph{Refining the Lipschitz constant(s)}

We have just discussed the LC of $V^*$ based on the LC of
finite-horizon strategies, reasoning on worst case Lipschitz constants
(one per time step) that hold for all strategies.
Now, (i) could we refine those constants based on knowledge regarding
$V^*$, in particular upper and lower bounds $U$ and $L$ (see next
sections)? And (ii) could we make use of those refined constants in the planning
process?

Regarding question (i), $U$ and $L$ tell us that any strategy profile
from time $\depth$ on
(and thus with remaining horizon $H$) has values within
$L^{\min}_\depth \eqdef \min_{\occ_\depth} L_\depth(\occ_\depth)$
and
$U^{\max}_\depth \eqdef \max_{\occ_\depth} U_\depth(\occ_\depth)$,
hence the refined Lipschitz constant:
\begin{align*}
  \l^{LU}_\depth & = \frac{U^{\max}_\depth-L^{\min}_\depth}{2}.
\end{align*}
Regarding question (ii), as $L$ and $U$ are refined during the
planning process, these refined depth-dependent constants would
progressively shrink, thus speeding up planning!
This phenomenon could encourage improving the value function bounds
where they seem high (for $U$) or low (for $L$).

\subsection{Approximating $V^*$}

Note: For the sake of readability, the depth index $\depth$ may be
omitted when it can be inferred from the occupancy state. 

\paragraph{Approximators}

An HSVI-like algorithm requires maintaining both
an upper and a lower approximator of $V^*$.
We denote them $U$ and $L$, and $\hat{V}=(L,U)$.

The LC of $V^*$ suggests employing LC function approximators for $U$
at depth $\depth$ ($U_\depth$) in the form of a lower envelope of (i)
an initial upper-bound $U^{(0)}(\occ)$ and (ii) downward-pointing
L1-cones, where an upper-bounding cone
$c^U_\omega = \langle \omega, u \rangle $---located at $\omega$, with
``summit'' value $u$, and slope $\l_{(H-\depth)}$---induces a function
$U^{(\omega)}(\occ) \eqdef u + \l_{(H-\depth)}
\norm{\omega-\occ}_{\p}$.
The upper-bound is thus defined as the lower envelope of $U^{(0)}$ and
the set of cones
$C^U_\depth = \{c^U_\omega\}_{\omega\in \Omega^U_\depth}$, \ie,
\begin{align*}
  U(\occ)
  & = \min\{ U^{(0)}(\occ), \min_{\omega\in \Omega^U_\depth} U^{(\omega)}(\occ) \}.
\end{align*}
Respectively, for the lower-bounding approximator at depth $\depth$:
a lower-bounding (upward-pointing) cone
$c^L_\omega = \langle \omega, l \rangle$ induces a function
$L_{\omega}(\occ)= l- \l_{(H-\depth)} \norm{\omega-\occ}_{\p}$; and
the lower bound is defined as the upper envelope of an initial lower
bound $L^{(0)}$ and the set of cones
$C^L_\depth = \{c^L_\omega\}_{\omega\in \Omega^L_\depth}$, \ie,
\begin{align*}
  L(\occ)
  & = \max\{ L^{(0)}(\occ), \max_{\omega\in \Omega^L_\depth} L^{(\omega)}(\occ) \}.
\end{align*}

\paragraph{(Point-based) Operator and Value Updates}

One cannot apply an operator (noted $\cH$) to update a value
function approximator uniformly. Instead, when visiting some occupancy state $\occ$ (at depth
$\depth \in \{0 \twodots H-1\}$), we perform a {\em point-based}
update of the upper-bound $U$ by
(i) finding the NEV of the following game (which relies on $U$ at
$\depth+1$):
\begin{align*}
  & U(\occ,\vbeta_\depth) \\
  & = \sum_{s,a^1,a^2} \left( \sum_\vth \occ(s,\vth) \beta^1(\theta^1,a^1) \beta^2(\theta^2,a^2) \right) r(s,a^1, a^2) \\
  & \quad + \gamma U\left( \nxt (\occ,\vbeta_\depth) \right)
\end{align*}
then (ii) adding a downward-pointing cone to $C^U_\depth$.
We note $\cK_\occ U$ the upper bound after this update at point
$\occ$.
The same applies to $L$ with upward-pointing cones instead, and using
notation $\cK_\occ L$.

\subsection{Initializations}
\label{sec|bound|init}

Due to the symmetry between players in a zs-POSG, without loss of
generality, let us look for an upper bound of the optimal value
function $V^*$, \ie, an optimistic bound (an admissible heuristic) for
(maximizing) player 1.
A usual approach to look for optimistic bounds is to relax the problem
for the player at hand.
To that end, one can here envision manipulating the players'
knowledge, their control over the system, the action ordering, or the
opponent's objective, \eg: \begin{enumerate}
\item providing more (\eg full) state observability to \player 1;
\item providing less (\eg no) state observability to \player 2;
\item letting \player 1 know what \player 2 observes;
\item letting \player 1 control chance (\player 2's choice would then
  only restrict the set of reachable states), but this would require
  that \player 1 has full observability; \item letting \player 2 act first, and telling \player 1 about \player
  2's selected action (exactly or through a partial observation);
\item turning \player 2 into a collaborator by making him maximize,
  rather than minimize, the expected return.
\end{enumerate}

Accounting for related Markov models for sequential decision-making,
this suggests turning the zs-POSG at hand for example into: \begin{itemize}
\item a Dec-POMDP by turning the opponent into a collaborator (or even
  into a POMDP or an MDP); or
\item a One-Sided POSGs \cite{HorBosPec-aaai17} by combining
(i) full state observability,
(ii) observability of \player 2's observation, and
(iii) observability of \player 2's action.
\end{itemize}
Note that making both players' actions or observations public (as in
PO-POSGs \cite{HorBos-aaai19}) would not be a viable solution as this
would imply providing more knowledge to both players at the same time,
which may prevent the resulting optimal value function from being an
upper bound for our problem.

\section{HSVI for zs-POSGs when $\gamma<1$}
\label{sec|HSVI}

In this section, unless stated otherwise, we consider
$\epsilon$-optimal $\gamma$-discounted problems under both finite and
infinite horizons.
The undiscounted (finite horizon) case will be treated separately.

\subsection{Algorithm}

As we shall see, $\epsilon$-optimally solving an $\infty$-horizon zs-POSG amounts, as
often, to solving a problem with finite horizon $H_{\max}$, which
allows exploiting the results derived up to now.
For convenience, we assume $H_{\max}$ already known and use
horizon-dependent constants (\eg, Lipschitz constants).

HSVI for zs-OMGs is detailed in Algorithm~\ref{alg|HSVI|zsOMG}.
As vanilla HSVI, it relies on
(i) generating trajectories while acting optimistically (lines
\ref{alg|hsvi|zsomg|nesU}--\ref{alg|hsvi|zsomg|nextS}), \ie, player
$1$ (resp. $2$) acting ``greedily'' w.r.t. $U$ (resp. $L$), and
(ii) locally updating the upper- and lower-bounding approximators
(lines \ref{alg|hsvi|update1} and \ref{alg|hsvi|update2}).
Here, computations of value updates and strategies rely on solving our
local zero-sum abnormal form games (possibly a maximin/minimax
optimization exploiting the Lipschitz continuity as discussed in
Sec.~\ref{sec|maximinimax}).
A key difference lies in the criterion for stopping trajectories.
In vanilla HSVI (for POMDPs), the finite branching factor allows
looking at the convergence of $U$ and $L$ at each point reachable
under an optimal strategy.
To ensure $\epsilon$-convergence at $b_0$, trajectories just need to
be interrupted when the current width at $b_\depth$
($\width(\hat{V}(b_\depth))$, where $\width(x,y)\eqdef y-x$) is smaller
than a threshold $\gamma^{-\depth} \epsilon$.
Here, dealing with an infinite branching factor, one may converge
towards an optimal solution while always visiting new points of the
occupancy space.
Yet, as the sequence of generated (deterministic) trajectories
converges to an optimal trajectory, the density of visited points
around it increases, so that the Lipschitz approximation error tends
to zero.
One can thus bound the width within balls around visited points by
exploiting the Lipschitz continuity of the optimal value function.
As proposed by Horák et al. \cite{HorBosPec-aaai17}, this is achieved by adding a
term $- \sum_{i=1}^\depth 2 \radius \l_{\depth} \gamma^{-i}$ to
ensure that the width is below $\gamma^{_\depth} \epsilon$ within a
ball of radius $\radius$ around the current point (here the occupancy
state $\occ_\depth$).
Hence the threshold
\begin{align}
  \label{eq|def|thr}
  \thr(\depth) & \eqdef \gamma^{-\depth}\epsilon - \sum_{i=1}^\depth 2 \radius \l_{\depth-i} \gamma^{-i}.
\end{align}

\begin{algorithm}
  \caption{zs-OMG-HSVI {\scriptsize (in \tcr{red}: differences with HSVI)}}\label{alg|HSVI|zsOMG}
  \DontPrintSemicolon

  \SetKwFunction{hsvifct}{\textbf{HSVI}}
  \SetKwFunction{explorefct}{\textbf{RecursivelyTry}}
  \SetKwFunction{mettreAJourfct}{\textbf{Update}}
  \SetKwFunction{widthfct}{\textbf{width}}

  \Fct{\hsvifct$(\epsilon)$}{
    Initialize $L$ and $U$\;
    \While{$\width(\occ_0)>\epsilon$}{
      \explorefct$(\occ_0, \depth=0)$
    }
    \Return{$L,U$}
  }

  \Fct{\explorefct$(\occ, \depth)$}{

    \If{$\width(\occ)> \tcr{\thr(\depth)}$ and $\depth<H$}{
      \mettreAJourfct$(\occ)$ \label{alg|hsvi|update1} \;
      \tcr{$\beta^U_\depth \gets \nes(\Gamma^\occ(U))$} \label{alg|hsvi|zsomg|nesU} \;
      \tcr{$\beta^L_\depth \gets \nes(\Gamma^\occ(L))$} \label{alg|hsvi|zsomg|nesL} \;
      \tcr{$\occ' \gets \nxt(\occ,\beta^{U,1}_\depth,\beta^{L,2}_\depth)$} \label{alg|hsvi|zsomg|nextS} \;
      \explorefct$(\occ',\depth+1)$\;
      \mettreAJourfct$(\occ)$ \label{alg|hsvi|update2} \;
    }
    \Return{}\;
  }

  \Fct{\mettreAJourfct$(\occ)$}{
    $L \gets $ \mettreAJourfct$(L,\occ)$ \tcc{\tcr{uses \rmfamily $\nev(\Gamma^\occ(L))$} \label{alg|hsvi|zsomg|majU}}
    $U \gets $ \mettreAJourfct$(U,\occ)$ \tcc{\tcr{uses \rmfamily $\nev(\Gamma^\occ(U))$\!\!\!} \label{alg|hsvi|zsomg|majL}}
  }
\end{algorithm}

\paragraph{Setting $\radius$}

As can be observed, this threshold function should always return
positive values, which requires a small enough $\radius$.
For a given problem, the maximum possible value $\radius$ shall depend
on the Lipschitz constants at each time step, which themselves depend
on the upper and lower bounds of the optimal value function (and thus
may evolve during the planning process).
For the sake of simplicity, let us consider a single Lipschitz
constant $\l$ common to all time steps, which always exists.

\begin{restatable}[Proof in App.~\extref{proofLemMaxRadius}]{lemma}{lemMaxRadius}
  \label{lem:MaxRadius}
  \IfAppendix{{\em (originally stated on
    page~\pageref{lem:MaxRadius})}}{}
Assuming a single depth-independent Lipschitz constant $\l$, and noting that
  \begin{align}
    \label{eq|thr}
    \thr(\depth)
    & = \gamma^{-\depth}\epsilon - 2 \radius \l \frac{\gamma^{-\depth}-1}{1-\gamma},
  \end{align}
  one can ensure positivity of the threshold at any $\depth \geq 1$ by
  enforcing
  $\radius < \frac{1}{2\l} \frac{1-\gamma}{1-\gamma^{H}} \epsilon$
  (with $\gamma^\infty=0$ in the case of an infinite horizon).
\end{restatable}

We shall thus pick $\rho$ in
$(0, \frac{1}{2\l} \frac{1-\gamma}{1-\gamma^{H}} \epsilon)$.
But what is the effect of setting $\radius$ to small or large values?
\begin{itemize}
\item The smaller $\radius$, the larger $\thr(\depth)$, the shorter
  the trajectories, but the smaller the balls and the higher the
  required density of points around the optimal trajectory, thus the
  more trajectories needed to converge.
\item The larger $\radius$, the smaller $\thr(\depth)$, the longer the
  trajectories, but the larger the balls and the lower the required
  density of points around the optimal trajectory, thus the less
  trajectories needed to converge.
\end{itemize}
So, setting $\radius$ means making a compromise between the number of
generated trajectories and their length (up to $H$ for finite horizon
problems).

\subsection{Finite-Time Convergence}

A first step towards proving the finite time convergence of the
algorithm is to bound, even in infinite horizon settings, the length
of HSVI's trajectories using the bounded width of $\hat{V}$ and the
exponential growth of $thr(\depth)$.

\begin{restatable}[Proof in App.~\extref{proofLemFiniteTrials}]{lemma}{lemFiniteTrials}
  \label{lem:finiteTrials}
  \IfAppendix{{\em (originally stated on
    page~\pageref{lem:finiteTrials})}}{}
Assuming a depth-independent Lipschitz constant $\l$, and with
  $\WUL \eqdef \norm{U^{(0)}-L^{(0)}}_\infty$, the length of
  trajectories is upper-bounded by
  \begin{align*}
    T_{\max}
    & \eqdef \ceil*{
      \log_{\gamma} \frac{
        \epsilon - \frac{2 \radius \l}{1-\gamma}
      }{
        \WUL - \frac{2 \radius \l }{1-\gamma}
      }
}.
  \end{align*}
\end{restatable}

Note that
(i) the classical upper-bound is retrieved when $\radius=0$ (Eq.~(6.7)
in \cite{Smith-phd07}), and
(ii) this gives us the maximum horizon $H_{\max}$ needed to solve the
problem.
In the case of a problem with finite horizon problem $H$, this means
that some trajectories may be shorter than $H$.
Now, knowing that any trial terminates in bounded time allows deriving
the following results, in order.

\begin{restatable}[Proof in App.~\extref{proofLemThr}]{theorem}{lemThr}
  \label{lem|thr}
  \IfAppendix{{\em (originally stated on
    page~\pageref{lem|thr})}}{}
Consider a trial $(\occ_0,\dots,\occ_\depth)$ of length $\depth$ and
  consider that the backward updates of $U_{\depth-1}$ and
  $L_{\depth-1}$ have {\bf not yet} been performed.
Then
  \begin{enumerate}
  \item $\width(\cK_{\occ_{\depth-1}}\hat{V}(\occ_{\depth-1})) \leq \thr(\depth-1) - 2 \radius \l_{\depth-1} $,
    and
  \item for every $\occ'_{\depth-1}$ satisfying
    $\norm{\occ'_{\depth-1}-\occ_{\depth-1}}_{\p} \leq \radius$, it holds: $\width(\cK_{\occ_{\depth-1}}\hat{V}(\occ'_{\depth-1})) \leq \thr(\depth-1).$
  \end{enumerate}
\end{restatable}

\begin{theorem}
  \label{thm|termination}
  Algorithm~\ref{alg|HSVI|zsOMG} terminates with an
  $\epsilon$-approximation of $V^*_0(\occ_0)$.
\end{theorem}

\begin{proof}(Adapted from \cite{HorBos-aaai19})
  Assume for the sake of contradiction that the algorithm does not
  terminate and generates an infinite number of explore trials.
Since the length of a trial is bounded by a finite number
  $T_{\max}$, the number of trials of length $T$ (for some
  $0 \leq T \leq T_{\max}$) must be infinite.
It is impossible to fit an infinite number of occupancy points
  $\occ_T$ satisfying $\norm{\occ_T-\occ'_T}_{\p} > \radius$ within $\Occ_T$.
Hence there must be two trials of length $T$,
  $\{\occ_{\depth,1}\}_{\depth=0}^T$ and $\{\occ_{\depth,2}\}_{\depth=0}^T$, such that
  $\norm{\occ_{T-1,1}-\occ_{T-1,2}}_{\p} \leq \radius$.
Without loss of generality, assume that $\occ_{T-1,1}$ was visited
  the first.
According to Lemma \ref{lem|thr}, the point-based update in
  $\occ_{T-1,1}$ resulted in
  $\width(\hat{V} (\occ_{T-1,2})) \leq \thr(T-1)$---which contradicts
  that the condition on line \ref{alg|hsvi|zsomg|nextS} of
  Algorithm~\ref{alg|HSVI|zsOMG} has not been satisfied for
  $\occ_{T-1,2}$ (and hence that $\{\occ_{t,2}\}_{t=0}^T$ was a trial
  of length $T$).
\end{proof}

Note that the number of trials could be (tediously) upper-bounded by
determining how many balls of radius $\radius$ are required to cover
occupancy simplexes at each depth.

\section{HSVI for zs-POSGs when $\gamma=1$}
\label{sec|HSVI2}

We now focus on 
$\epsilon$-optimally solving finite horizon problems under the total criterion ($\gamma=1$).

\subsection{Algorithm}

The algorithm remains identical to Algorithm~\ref{alg|HSVI|zsOMG}, up
to update operators (which just use $\gamma=1$), but for the the
threshold function used for prematurely terminating trajectories, with
\begin{align}
  \label{eq|def|thr2}
  \thr(\depth) & \eqdef \epsilon - \sum_{i=1}^\depth 2 \radius \l_{\depth-i}.
\end{align}

\paragraph{Setting $\radius$}

As in the discounted case, this threshold function should always
return positive values, which requires a small enough $\radius$.
Let us now upper-bound the maximum possible value $\radius$ using the
simple upper-bounds of the Lipschitz constants:
$\l_\depth=(H-\depth)\cdot(r_{\max}-r_{\min})$.

\begin{restatable}[Proof in App.~\extref{proofLemMaxRadiusBis}]{lemma}{lemMaxRadiusBis}
  \label{lem|MaxRadius|bis}
  \IfAppendix{{\em (originally stated on
      page~\pageref{lem|MaxRadius|bis})}}{}
Using $\l_\depth = (H-\depth) \cdot ( r_{\max} - r_{\min} )$ for any
  $\depth$, and noting that
  \begin{align}
    \label{eq|thr|bis}
    \thr(\depth)
    & = \epsilon - \radius (r_{\max}-r_{\min}) \left[ (2H + 1 - \depth) \depth \right],
  \end{align}
  one can ensure positivity of the threshold at any
  $\depth \in \{1, \dots, H-1 \}$ by enforcing
  $\frac{\epsilon}{(r_{\max}-r_{\min}) (H + 1)H}$.
\end{restatable}

\uline{Side note:} Had we used depth-independent $\l=H(r_{\max}-r_{\min})$, we would get
\begin{align*}
  thr(\depth) & = \epsilon - 2 \radius (r_{\max}-r_{\min}) H \depth, \text{ and} \\
  \radius & < \frac{\epsilon}{2 (r_{\max}-r_{\min}) H^2 },
\end{align*}
which would make for an $\sim$2 times smaller maximum $\radius$, but a
bound more similar in shape to the one obtained for $\gamma<1$.

We shall thus pick $\rho$ in
$(0, \frac{\epsilon}{(r_{\max}-r_{\min}) (H + 1)H})$.
But what is the effect of setting $\radius$ to small or large values?
\begin{itemize}
\item The smaller $\radius$, the larger $\thr(\depth)$, the shorter
  the trajectories, but the smaller the balls and the higher the
  required density of points around the optimal trajectory, thus the
  more trajectories needed to converge.
\item The larger $\radius$, the smaller $\thr(\depth)$, the longer the
  trajectories, but the larger the balls and the lower the required
  density of points around the optimal trajectory, thus the less
  trajectories needed to converge.
\end{itemize}
So, setting $\radius$ means making a compromise between the number of
generated trajectories and their length (up to $H$ for finite horizon
problems).

\subsection{Finite-Time Convergence}

In this finite horizon, each trajectory terminates in at most $H$ steps.
Now, knowing that any trial terminates in bounded time allows deriving
the following results, in order.

\begin{restatable}[Proof in App.~\extref{proofLemThrBis}]{theorem}{lemThrBis}
  \label{lem|thr|bis}
  \IfAppendix{{\em (originally stated on
    page~\pageref{lem|thr|bis})}}{}
Consider a trial $(\occ_0,\dots,\occ_\depth)$ of length $\depth$ and
  consider that the backward updates of $U_{\depth-1}$ and
  $L_{\depth-1}$ have {\bf not yet} been performed.
Then
  \begin{enumerate}
  \item $\width(\cK_{\occ_{\depth-1}}\hat{V}(\occ_{\depth-1})) \leq \thr(\depth-1) - 2 \radius \l_\depth $,
    and
  \item for every $\occ'_{\depth-1}$ satisfying
    $\norm{\occ'_{\depth-1}-\occ_{\depth-1}}_{\p} \leq \radius$, it holds: $\width(\cK_{\occ_{\depth-1}}\hat{V}(\occ'_{\depth-1})) \leq \thr(\depth-1).$
  \end{enumerate}
\end{restatable}

\begin{theorem}
  \label{thm|termination|bis}
  Algorithm~\ref{alg|HSVI|zsOMG} terminates with an
  $\epsilon$-approximation of $V^*_0(\occ_0)$.
\end{theorem}

\begin{proof}(Adapted from \cite{HorBos-aaai19})
  Assume for the sake of contradiction that the algorithm does not
  terminate and generates an infinite number of explore trials.
Since the length of a trial is bounded by a finite number
  $T_{\max}$, the number of trials of length $T$ (for some
  $0 \leq T \leq T_{\max}$) must be infinite.
It is impossible to fit an infinite number of occupancy points
  $\occ_T$ satisfying $\norm{\occ_T-\occ'_T}_{\p} > \radius$ within $\Occ_T$.
Hence there must be two trials of length $T$,
  $\{\occ_{\depth,1}\}_{\depth=0}^T$ and $\{\occ_{\depth,2}\}_{\depth=0}^T$, such that
  $\norm{\occ_{T-1,1}-\occ_{T-1,2}}_{\p} \leq \radius$.
Without loss of generality, assume that $\occ_{T-1,1}$ was visited
  the first.
According to Lemma \ref{lem|thr}, the point-based update in
  $\occ_{T-1,1}$ resulted in
  $\width(\hat{V} (\occ_{T-1,2})) \leq \thr(T-1)$---which contradicts
  that the condition on line \ref{alg|hsvi|zsomg|nextS} of
  Algorithm~\ref{alg|HSVI|zsOMG} has not been satisfied for
  $\occ_{T-1,2}$ (and hence that $\{\occ_{t,2}\}_{t=0}^T$ was a trial
  of length $T$).
\end{proof}

Note that the number of trials could be (tediously) upper-bounded by
determining how many balls of radius $\radius$ are required to cover
occupancy simplexes at each depth.

\section{Discussion}

Inspired by techniques solving POMDPs as belief MDPs or Dec-POMDPs as
occupancy MDPs, we have demonstrated that zs-POSGs could be turned
into a new type of sequential game, namely zs-OMGs, allowing to apply
Bellman's optimality principle.
Value function approximators (with heuristic initializations) can be
used thanks to the Lipschitz continuity of $V^*$, and despite $V^*$
possibly {\bf not} being concave or convex in any relevant statistic.
A variant of HSVI has been derived which provably converges in finite
time to an $\epsilon$-optimal solution.

This approach was motivated by the fact that the corresponding
techniques for POMDPs and Dec-POMDPs provide state-of-the-art solvers.
The time complexity of the algorithm shall depend, among other things,
on that of the maximin/minimax optimization technique in use, and on
how many trials are required before convergence.
We also currently lack empirical comparisons of the resulting
algorithm with existing zs-POSG solution techniques.

Several implementation details could be further discussed as
the maximin/minimax error-bounded optimization algorithm,
the need to regularly prune dominated cones in $U$ and $L$, and
the possible use of compression techniques to reduce the
dimensionality of the occupancy subspaces, as in FB-HSVI
\cite{DibAmaBufCha-jair16}.

Regarding execution, as in single-agent or collaborative multi-agent
settings, while exploration is guided by optimistic decisions
(greediness w.r.t. $U$ for \player $1$ and $L$ for \player $2$), actual
decisions should be pessimistic, \ie, \player $1$ should act
``greedily'' w.r.t. $L$, and $2$ w.r.t. $U$.

Handling finite-horizon settings requires little changes.
The maximum length of trials shall be the minimum between this horizon
and the bound that depends on $\epsilon$ and $\radius$.
Additionally considering $\gamma=1$ shall require revising the
Lipschitz constants and some other formulas.

As often with Dec-POMDPs \cite{SzeChaZil-uai05,DibAmaBufCha-jair16},
each player's strategy is here history-dependent, because one could
not come up with private belief states, which is feasible under
certain assumptions \cite{HorBosPec-aaai17,HorBos-aaai19}.
One could possibly address this issue as MacDermed and Isbell \cite{MacIsb-nips13} did by
assuming that a bounded number of beliefs is sufficient to solve the
problem.

Public actions and observations, as in Poker, could be exploited by
turning the non-observable sequential decision problem faced by the
central planner into a partially observable one, and thus the
deterministic OMG into a probabilistic one.

\bibliographystyle{abbrv}

\ifextended{
  \newpage
  \onecolumn
  \appendix

\section{Appendix}

This appendix mainly provides proofs of several theoretical claims of
the paper.

\subsection{From zs-POSGs to zs-OMGs}
\label{app:fromTo}

The following result shows that the occupancy state is Markovian, \ie,
its value at $\depth$ only depends on its previous value
($\occ_{\depth-1}$), the system dynamics ($P^{z^1,z^2}_{a^1,a^2}$),
and the last behavioral decision rules ($\beta^1_{\depth-1}$ and
$\beta^2_{\depth-1}$).

\begin{lemma}
  \label{lem:nextOcc}
  Given an occupancy state $\occ_{\depth-1}$ and a behavioral decision rule
  profile $\beta_ {\depth-1}$, next occupancy state $\occ_\depth$ is given by
  the following formula (for any $s'$, $\theta^1_{\depth-1}$, $a^1$, $z^1$,
  $\theta^2_{\depth-1}$, $a^2$, $z^2$):
  \begin{align*}
    \occ_\depth(s',(\theta^1_{\depth-1},a^1,z^1),(\theta^2_{\depth-1},a^2,z^2))
    & 
    = \beta^1_{\depth-1}(\theta^1_{\depth-1},a^1) \cdot \beta^2_{\depth-1}(\theta^2_{\depth-1},a^2) 
    \sum_{s} P^{z^1,z^2}_{a^1,a^2}(s'|s) 
    \cdot \occ_{\depth-1} (s,\theta^1_{\depth-1},\theta^2_{\depth-1}).
  \end{align*}
\end{lemma}

\begin{proof}
  The proof goes by simply developing the definition:
  \begin{align*}
    \occ_\depth(s',(\theta^1_{\depth-1},a^1,z^1),(\theta^2_{\depth-1},a^2,z^2)) 
    & \eqdef Pr(s',(\theta^1_{\depth-1},a^1,z^1),(\theta^2_{\depth-1},a^2,z^2)) \\
    & = \sum_{s} Pr(s,s',(\theta^1_{\depth-1},a^1,z^1),(\theta^2_{\depth-1},a^2,z^2)) \\
    & = \sum_{s} Pr(s',z^1,z^2 | s,\theta^1_{\depth-1},a^1,\theta^2_{\depth-1},a^2)
    \cdot Pr(s,\theta^1_{\depth-1},a^1,\theta^2_{\depth-1},a^2) \\
    & = \sum_{s} P^{z^1,z^2}_{a^1,a^2}(s'|s) 
    \cdot Pr(a^1,a^2 |s,\theta^1_{\depth-1},\theta^2_{\depth-1})
    \cdot Pr(s,\theta^1_{\depth-1},\theta^2_{\depth-1}) \\
    & = \sum_{s} P^{z^1,z^2}_{a^1,a^2}(s'|s) 
    \cdot Pr(a^1 |s,\theta^1_{\depth-1},\theta^2_{\depth-1})
    \cdot Pr(a^2 |s,\theta^1_{\depth-1},\theta^2_{\depth-1}) 
    \cdot Pr(s,\theta^1_{\depth-1},\theta^2_{\depth-1}) \\
    & = \sum_{s} P^{z^1,z^2}_{a^1,a^2}(s'|s) 
    \cdot \beta^1_{\depth-1}(\theta^1_{\depth-1},a^1) \cdot \beta^2_{\depth-1}(\theta^2_{\depth-1},a^2) 
    \cdot \occ_{\depth-1} (s,\theta^1_{\depth-1},\theta^2_{\depth-1}) \\
    & = \beta^1_{\depth-1}(\theta^1_{\depth-1},a^1) \cdot \beta^2_{\depth-1}(\theta^2_{\depth-1},a^2) 
    \sum_{s} P^{z^1,z^2}_{a^1,a^2}(s'|s) 
    \cdot \occ_{\depth-1} (s,\theta^1_{\depth-1},\theta^2_{\depth-1}).
    \qedhere
  \end{align*}
\end{proof}

\subsection{Back to Mixed Strategies}
\label{proofLemEquivalence}

The following result demonstrate that, instead of reasoning on mixed
strategies constrained to be ``compatible'' with some occupancy state
$\occ_\depth$, one can reason equivalently with behavioral strategies.

\lemConvexity*

\begin{proof}
  Let
$\mu^{i,1}_{0:H-1|\occ_\depth\rangle}$ and
  $\mu^{i,2}_{0:H-1|\occ_\depth\rangle}$ be two mixed strategies of
  $i$ in $M^i_{0:H-1|\occ_\depth\rangle}$; and
$\alpha \in [0,1]$.
We want to show that the mixed strategy
  $\nu^i_{0:H-1} \eqdef \alpha \cdot \mu^{i,1}_{0:H-1|\occ_\depth\rangle} +
  (1-\alpha) \cdot \mu^{i,2}_{0:H-1|\occ_\depth\rangle}$
  is also compatible with $\occ_\depth$.
As for behavioral strategies, this holds if, for any history-action
  pair $(\theta^i_t,a^i)$ (with $t < \depth$),
\begin{align*}
    \text{either }
    Pr(\theta^i_t|\occ_\depth)
    & = 0, \\
    \text{or } 
    Pr(a^i|\nu^i_{0:H-1},\theta^i_t)
    & = Pr(a^i|\occ_\depth,\theta^i_t).
  \end{align*}
  Let us thus consider the case where $Pr(\theta^i_t|\occ_\depth) \neq 0$:
  \begin{align*}
    Pr(a^i|\nu^i_{0:H-1},\theta^i_t)
    & = \sum _{\pi^i_{0:H-1} \in \Pi^i_{0:H-1}(\theta^i_t,a^i)} \nu^i(\pi^i_{0:H-1}) \\
    & = \sum _{\pi^i_{0:H-1} \in \Pi^i_{0:H-1}(\theta^i_t,a^i)} \left(
      \alpha \cdot \mu^{i,1}_{0:H-1|\occ_\depth\rangle}(\pi^i_{0:H-1}) +
      (1-\alpha) \cdot \mu^{i,2}_{0:H-1|\occ_\depth\rangle}(\pi^i_{0:H-1}) 
      \right) \\
    & = 
      \alpha \cdot \Big( \sum_{\substack{\pi^i_{0:H-1} \in \\ \Pi^i_{0:H-1}(\theta^i_t,a^i)}} \mu^{i,1}_{0:H-1|\occ_\depth\rangle}(\pi^i_{0:H-1}) \Big) +
      (1-\alpha) \cdot \Big( \sum_{\substack{\pi^i_{0:H-1} \in \\ \Pi^i_{0:H-1}(\theta^i_t,a^i)}} \mu^{i,2}_{0:H-1|\occ_\depth\rangle}(\pi^i_{0:H-1}) \Big) \\
    & = 
      \alpha \cdot Pr(a^i|\mu^{i,1}_{0:H-1|\occ_\depth\rangle},\theta^i_t) +
      (1-\alpha) \cdot Pr(a^i|\mu^{i,2}_{0:H-1|\occ_\depth\rangle},\theta^i_t) \\
    & = 
      \alpha \cdot Pr(a^i|\occ_\depth,\theta^i_t) +
      (1-\alpha) \cdot Pr(a^i|\occ_\depth,\theta^i_t) \\
    & = Pr(a^i|\occ_\depth,\theta^i_t).
  \end{align*}
  As this holds for any $(\theta^i_t,a^i)$ that has non-zero
  probability in $\occ_\depth$, then indeed a convex combination of
  two mixed strategies compatible with $\occ_\depth$ is also
  compatible.
$M^i_{0:H-1|\occ_\depth\rangle}$ is thus a convex set.
\end{proof}

\subsection{Nash Equilibria in Subgames?}
\label{app:subNEs}

The definition of Nash equilibrium still applies in each of the
infinitely many (induced) subgames and, as explained by the following
lemma, NEs are ``nested''.

\begin{lemma}
  \label{lem:nestedNEs}
  A Nash equilibrium of $V_\depth( \occ_\depth, \cdot, \cdot)$ induces
  a Nash equilibrium in any (induced) subgame
  $V_{\depth'}( \occ_{\depth'}, \cdot, \cdot)$ for
  $\depth'\geq \depth$.
\end{lemma}

\begin{proof}
  For any $\depth \in \{ 0 \twodots H-1 \}$ and any $\occ_\depth$,
(i) $V_\depth( \occ_\depth, \mu^1, \mu^2)$ is linear in both $\mu^1$
  and $\mu^2$, and
(ii) the space of mixed strategy profiles constrained by
  $\occ_\depth$ is convex (Lemma~\ref{lem:convexity}).
This allows applying von Neumann's minimax theorem, so that this
  subgame at $\occ_\depth$ induces a 2-player zero-sum normal-form game for which at least one NE exists
  (and all NEs are equivalent).

  Lemma~\ref{lem:convexity} (p.~\pageref{lem:convexity}) allows
  reasoning with behavorial strategy profiles instead of mixed ones,
  and we can focus on what happens from $\depth$ on.
So, let $\vbeta^{*}_{\depth:H-1}$ be a NES of
  $V_{\depth}(\occ_\depth,\cdot,\cdot)$.
Then the definition of Nash equilibrium allows writing in
  particular, $\forall \beta^1_{\depth+1:H-1}$:
  \begin{align*}
    V_{\depth}(\occ_\depth,
    \langle \beta^{1,*}_{\depth} \oplus \beta^{1,*}_{\depth+1:H-1} \rangle,
    \beta^{2,*}_{\depth:H-1}) & \geq
    \ V_{\depth}(\occ_\depth,
      \langle \beta^{1,*}_{\depth} \oplus \beta^1_{\depth+1:H-1} \rangle,
      \beta^{2,*}_{\depth:H-1}),
    \intertext{thus,}
    r(\occ_\depth, \beta^{1,*}_{\depth}, \beta^{2,*}_{\depth}) +
      \gamma V_{\depth+1}( \nxt(\occ_\depth, \beta^{1,*}_{\depth}, \beta^{2,*}_{\depth}), 
      \beta^{1,*}_{\depth+1:H-1}, \beta^{2,*}_{\depth+1:H-1}) \hspace{-3cm} \\
    & \geq
     \ r(\occ_\depth, \beta^{1,*}_{\depth}, \beta^{2,*}_{\depth}) +
      \gamma V_{\depth+1}( \nxt(\occ_\depth, \beta^{1,*}_{\depth}, \beta^{2,*}_{\depth}), 
      \beta^{1}_{\depth+1:H-1}, \beta^{2,*}_{\depth+1:H-1})
    \intertext{or, equivalently,}
    V_{\depth+1}( \nxt(\occ_\depth, \beta^{1,*}_{\depth}, \beta^{2,*}_{\depth}), 
    \beta^{1,*}_{\depth+1:H-1}, \beta^{2,*}_{\depth+1:H-1})
    & \geq
    \ V_{\depth+1}( \nxt(\occ_\depth, \beta^{1,*}_{\depth}, \beta^{2,*}_{\depth}), 
    \beta^{1}_{\depth+1:H-1}, \beta^{2,*}_{\depth+1:H-1}).
  \end{align*}
  With the symmetric property holding for player $2$, this implies that
  $\vbeta^*_{\depth+1:H-1} \eqdef
  (\beta^{1,*}_{\depth+1:H-1},
  \beta^{2,*}_{\depth+1:H-1})$
  is a NES of the {\em constrained} 2-player zero-sum (normal-form)
  game
  $V_{\depth+1}( \nxt(\occ_\depth,\vbeta^*_{\depth}),\cdot,\cdot)$.
By induction, a NES is obtained for any subgame
  $V_{\depth'}( \occ_{\depth'}, \cdot, \cdot)$ ($\depth'\geq \depth$).
\end{proof}

In particular, as expected, any Nash equilibrium solution of our
original game $V_0(\occ_0,\cdot,\cdot)$ induces a Nash equilibrium in
any of its reachable subgames, ensuring a rational behavior at each
time step.

\subsection{Solving POSGs as Occupancy Markov Games}

\subsubsection{An Abnormal-Form Game?}
\label{appLinearity}

The next two lemmas lead to demonstrating that
$V_\depth(\occ_\depth,\langle \cdot, \cdot\rangle)$ is linear in
$\beta^i_{\depth'}$ for $i\in \{1,2\}$ and $\depth'\geq \depth$.

\begin{lemma}
  \label{lem|occ|lin|dr}
  At depth $\depth$, $\nxt(\occ_\depth,\vbeta_\depth)$ is
  linear in $\beta^1_\depth$ and $\beta^2_\depth$, where
  $\vbeta_\depth=\langle \beta^1_\depth, \beta^2_\depth\rangle$.
\end{lemma}

\begin{proof}
  Let $\occ_\depth$ be an occupancy state at depth $\depth$ and
  $\vbeta_\depth$ be a decision rule.
Then the next occupancy state
  $\tilde{\occ}=\nxt(\occ_\depth,\vbeta_\depth)$ satisfies, for any
  $\tilde{s}$ and $(\vth,\va,\vz)$:
  \begin{align*}
    \tilde{\occ}(\tilde{s},(\vth,\va,\vz)) 
    & = \sum_{s\in \cS, \vth \in \vTh} \occ_{\depth}(s,\vth) \vbeta_\depth(\vth,\va) \PP{s}{\va}{s'}{\vz} \\
    & = \sum_{s\in \cS, \vth \in \vTh} \occ_{\depth}(s,\vth) \beta^1_\depth(\theta^1,a^1) \beta^2_\depth(\theta^2,a^2) \PP{s}{a^1,a^2}{s'}{z^1,z^2} \\
    & = \sum_{\vth \in \vTh} \beta^1_\depth(\theta^1,a^1) \beta^2_\depth(\theta^2,a^2) \left( \sum_{s\in \cS} \occ_{\depth}(s,\vth) \PP{s}{a^1,a^2}{s'}{z^1,z^2} \right).
  \end{align*}
  The next occupancy state thus also evolves linearly w.r.t. {\em
    private} decision rules at $\depth$. \end{proof}

\begin{lemma}
  \label{lem|r|lin|dr}
  At depth $\depth$, $r_\depth(\occ_\depth,\vbeta_\depth)$ is
  linear in $\beta^1_\depth$ and $\beta^2_\depth$, where
  $\vbeta_\depth=\langle \beta^1_\depth, \beta^2_\depth\rangle$.
\end{lemma}

\begin{proof}
  When visiting some occupancy state $\occ_\depth$ (at depth
  $\depth \in \{0, \dots, \depth_{\max}-1\}$), for joint behavioral
  decision rule $\vbeta_\depth = \langle \beta^1, \beta^2 \rangle$,
  the expected immediate reward at $\depth$ is:
  \begin{align*}
    r_\depth(\occ_\depth,\vbeta_\depth) 
    & = \sum_{s,a^1,a^2} \left( \sum_\vth \occ_\depth(s,\vth) \beta^1(\theta^1,a^1) \beta^2(\theta^2,a^2) \right) r(s,a^1, a^2) \\
    & = \sum_{\theta^1, \theta^2} \sum_{a^1,a^2} \beta^1(\theta^1,a^1) \beta^2(\theta^2,a^2) \left( \sum_s \occ_\depth(s,\vth) r(s,a^1, a^2) \right),
  \end{align*}
  \ie, it is linear in $\beta^1$ as well as $\beta^2$ (hence bilinear).
\end{proof}

\begin{corollary}
  \label{cor|V|lin|dr}
  At depth $\depth$,
  $V_\depth(\occ_\depth, \vbeta_{\depth:H-1})$ is linear in
  $\beta^1_{\depth'}$ as well as $\beta^2_{\depth'}$ for any
  $\depth' \in \{\depth \twodots H-1 \}$.

\end{corollary}

\begin{proof}
  First, the property trivially holds for $\depth=H-1$.

  Let us now assume that it holds for some $\depth+1 \in \{1\twodots H-1\}$.
Then, we have at $\depth$:
  \begin{align*}
    V_\depth(\occ_\depth,\vbeta_{\depth:H-1})  
    & \specialcell{
      = r_\depth(\occ_\depth,\vbeta_\depth) 
      + \gamma V_{\depth+1}\left( \nxt (\occ,\vbeta_\depth), \vbeta_{\depth+1:H-1} \right)
      \hfill \text{\{using Lemma~\ref{lem|r|lin|dr}:\}}} \\
& = \left[ \sum_{\vth,a^1,a^2} \beta^1_\depth(\theta^1,a^1) \beta^2_\depth(\theta^2,a^2) \left( \sum_s \occ(s,\vth) r(s,a^1, a^2) \right) \right] 
      + \gamma V_{\depth+1}\left( \nxt (\occ,\vbeta_\depth), \vbeta_{\depth+1:H-1} \right) .
  \end{align*}
  As
  \begin{itemize}
  \item $\nxt(\occ_\depth,\vbeta_\depth)$ is linear in decision rules
    $\beta^1_\depth$ and $\beta^2_\depth$ (Lemma~\ref{lem|occ|lin|dr}) and
  \item $V_{\depth+1}(\occ_{\depth+1}, \vbeta_{\depth+1:H-1})$
    is linear in $\occ_{\depth+1}$,
  \end{itemize}
  then, by composition (and other basic combinations),
  $V_\depth(\occ_\depth,\vbeta_{\depth:H-1})$ is linear in
  decision rules $\beta^1_\depth$ and $\beta^2_\depth$.

  Also, for any $\depth' > \depth$, the first term (expected instant
  reward) is independent of $\beta_{\depth'}$, and the second term is
  linear in $\beta^1_{\depth'}$ and $\beta^2_{\depth'}$ (by induction
  hypothesis), so that $V_\depth(\occ_\depth,\vbeta_{\depth:H-1})$ is
  linear in $\beta^1_{\depth'}$ and $\beta^2_{\depth'}$.

  Repeating this process, by induction the property holds for all
  $\depth \in \{0\twodots H-1\}$.

\end{proof}

One issue is that $V_\depth(\occ_\depth,\vbeta_{\depth:H-1})$ is not
linear in $\beta^1_{\depth:H-1}$ but multi-linear in
$\beta^1_{\depth'}$ for all $\depth' \geq \depth$ (idem for player
$2$).
As a consequence, this function may not be convex in
$\beta^1_{\depth:H-1}$ (or concave in
$\beta^2_{\depth:H-1}$).

\subsubsection{Properties of the Maximin and Minimax Values}

The next two results demonstrate that solving
$Q^*(\occ_\depth,\vbeta_\depth)$ for maximin and minimax values allows
finding one Nash equilibrium strategy profile (NES), so that Bellman's
optimality principle can be applied.

\lemAbnormalMaximinimax*

\begin{proof}
  \label{proofLemAbnormalMaximinimax}

Focusing, without loss of generality, on player $1$, we have
  (complementary explanations follow for numbered lines in
  particular):
  \begin{align}
    maximin(\occ_\depth) 
    & \eqdef \max_{\beta^1_\depth} \min_{\beta^2_\depth} Q^*_\depth(\occ_\depth,\beta^1_\depth,\beta^2_\depth) \nonumber \\
    & = \max_{\beta^1_\depth} \min_{\beta^2_\depth} \left[ r(\occ_\depth, \beta^1_\depth, \beta^2_\depth)
      + \gamma V^*_{\depth+1}(\nxt(\occ_\depth,\beta^1_\depth, \beta^2_\depth)) \right] \nonumber \\
    \intertext{($V^*_{\depth+1}(\nxt(\occ_\depth,\beta^1_\depth, \beta^2_\depth))$ being the Nash equilibrium value of normal-form game $V_{\depth+1}(\nxt(\occ_\depth,\beta^1_\depth, \beta^2_\depth), \mu^1, \mu^2 )$:)} & = \max_{\beta^1_\depth} \min_{\beta^2_\depth} \left[ r(\occ_\depth, \beta^1_\depth, \beta^2_\depth)
      + \gamma
      \max_{\mu^1\in M^1_{| \nxt(\occ_\depth, \vbeta_\depth) \rangle}}
      \min_{\mu^2\in M^2_{| \nxt(\occ_\depth, \vbeta_\depth) \rangle}}
      V_{\depth+1}(\nxt(\occ_\depth,\beta^1_\depth, \beta^2_\depth), \mu^1, \mu^2 ) \right] \nonumber \\
    \intertext{(there is no loss in precising how some occupancy state has been reached:)}
    & = \max_{\beta^1_\depth} \min_{\beta^2_\depth} \left[ r(\occ_\depth, \beta^1_\depth, \beta^2_\depth)
      + \gamma
      \max_{\mu^1\in M^1_{| \occ_\depth, \vbeta_\depth \rangle}}
      \min_{\mu^2\in M^2_{| \occ_\depth, \vbeta_\depth \rangle}}
      V_{\depth+1}(\nxt(\occ_\depth,\beta^1_\depth, \beta^2_\depth), \mu^1, \mu^2 ) \right] \nonumber \\
    \intertext{(the knowledge of $\beta^{-i}_\depth$ just tells $i$ which histories are irrelevant, and thus can be ignored:)}
    & = \max_{\beta^1_\depth} \min_{\beta^2_\depth} \left[ r(\occ_\depth, \beta^1_\depth, \beta^2_\depth)
      + \gamma
      \max_{\mu^1\in M^1_{| \occ_\depth, \beta^1_\depth \rangle}}
      \min_{\mu^2\in M^2_{| \occ_\depth, \beta^2_\depth \rangle}}
      V_{\depth+1}(\nxt(\occ_\depth,\beta^1_\depth, \beta^2_\depth), \mu^1, \mu^2 ) \right]  \label{eq|ignoring} \\
& = \max_{\beta^1_\depth} \min_{\beta^2_\depth}
      \max_{\mu^1\in M^1_{| \occ_\depth, \beta^1_\depth\rangle}}
      \min_{\mu^2\in M^2_{| \occ_\depth, \beta^2_\depth\rangle}}
      \left[ r(\occ_\depth, \beta^1_\depth, \beta^2_\depth)
      + \gamma V_{\depth+1}(\nxt(\occ_\depth,\beta^1_\depth, \beta^2_\depth), \mu^1, \mu^2 ) \right] \nonumber \\
\intertext{(using the equivalence between maximin and minimax values for the (constrained normal-form) game at $\depth+1$, the last two max and min operators can be swapped:)}
    & = \max_{\beta^1_\depth} \min_{\beta^2_\depth}
      \min_{\mu^2\in M^2_{| \occ_\depth, \beta^2_\depth \rangle}}
      \max_{\mu^1\in M^1_{| \occ_\depth, \beta^1_\depth \rangle}}
      \left[ r(\occ_\depth, \beta^1_\depth, \beta^2_\depth)
      + \gamma V_{\depth+1}(\nxt(\occ_\depth,\beta^1_\depth, \beta^2_\depth), \mu^1, \mu^2) \right] \nonumber \\
\intertext{(merging both mins and observing that decision rule $\beta^2_\depth$
    at time $\depth$ can be retrieved as a function of $\mu^2$ (noted
    $\beta^2_\depth(\mu^2)$):)}
& = \max_{\beta^1_\depth}
      \min_{\mu^2 \in M^2_{| \occ_\depth \rangle}}
      \max_{\mu^1 \in M^1_{| \occ_\depth, \beta^1_\depth \rangle}}
      \left[ r(\occ_\depth, \beta^1_\depth, \beta^2_\depth(\mu_2))
      + \gamma V_{\depth+1}(\nxt(\occ_\depth,\beta^1_\depth), \mu^1, \mu^2 ) \right]  \nonumber \\ \intertext{(using again the minimax theorem's equivalence between maximin and minimax on an appropriate game:)}
    & = \max_{\beta^1_\depth}
      \max_{\mu^1\in M^1_{| \occ_\depth, \beta^1_\depth\rangle}}
      \min_{\mu^2\in M^2_{| \occ_\depth\rangle}}
      \left[ r(\occ_\depth, \beta^1_\depth, \beta^2_\depth(\mu_2))
      + \gamma V_{\depth+1}(\nxt(\occ_\depth,\beta^1_\depth, \beta^2_\depth(\mu_2)), \mu^1, \mu^2 ) \right] \label{eq|appropriateGame} \\
    \intertext{(merging both maxs and observing that decision rule $\beta^1_\depth$
    at time $\depth$ can be retrieved as a function of $\mu^1$ (noted
    $\beta^1_\depth(\mu^1)$):)}
& =
      \max_{\mu^1\in M^1_{| \occ_\depth\rangle}}
      \min_{\mu^2\in M^2_{| \occ_\depth\rangle}}
      \left[ r(\occ_\depth, \beta^1_\depth(\mu_1), \beta^2_\depth(\mu_2))
      + \gamma V_{\depth+1}(\nxt(\occ_\depth,\beta^1_\depth(\mu^1), \beta^2_\depth(\mu_2)), \mu^1, \mu^2 ) \right] \nonumber \\ \intertext{(again with the equivalence property discussed before the lemma:)}
    & =
      \max_{\mu^1\in M^1_{| \occ_\depth\rangle}}
      \min_{\mu^2\in M^2_{| \occ_\depth\rangle}}
      V_\depth(\occ_\depth, \mu^1, \mu^2) \nonumber \\
    & =
      \max_{\beta^1_{\depth:H-1 | \occ_\depth\rangle}}
      \min_{\beta^2_{\depth:H-1 | \occ_\depth\rangle}}
      V_\depth(\occ_\depth, \beta^1_{\depth:H-1}, \beta^2_{\depth:H-1}) \nonumber \\
    & \eqdef
      V^*_\depth(\occ_\depth). \nonumber
  \end{align}

  \uline{Line \ref{eq|ignoring}} is obtained by observing that the
  knowledge of $\beta^{-i}_\depth$ only allows $i$ to ignore some
  irrelevant histories, thus reducing the size of the search space,
  but does not influence the expected return.
  
  \uline{Line \ref{eq|appropriateGame}} results from the observation
  that, while $M^1_{| \occ_\depth, \beta^1_\depth \rangle}$ and
  $M^2_{| \occ_\depth \rangle}$ allow to actually make decision over
  different time intervals, we are here minimizing over $\mu^2$ while
  maximizing over $\mu^1$ over a function that is linear in both input
  spaces.
This amounts to solving some 2-player zero-sum normal-form game,
  hence the applicability of von Neumann's minimax theorem.

  The above derivation tells us that the maximin value (the best
  outcome player $1$ can guarantee whatever player $2$'s strategy) in
  the one-time-step game is thus the Nash equilibrium value (NEV) for
  the complete subgame from $\depth$ onwards.
\end{proof}

\corAbnormalMaximinimax*

\begin{proof}
\label{proofCorAbnormalMaximinimax}
  When player $1$ (resp. $2$) selects a strategy guaranteeing the
  maximin (resp. minimax) value, the same value is guaranteed for both
  players, so that none of them can do better by opting for a
  different strategy.
This situation is thus, by definition, a Nash equilibrium.
\end{proof}

\subsubsection{Maximin and Minimax Computation}
\label{proofLemOccLin}
\label{proofLemQLC}

The next two results demonstrate the Lipschitz-continuity of
$Q^*_\depth(\occ_\depth,\beta^1_\depth,\beta^2_\depth)$ in both
$\beta^1_\depth$ and $\beta^2_\depth$, which will allow finding
$\epsilon$-optimal solutions of the maximin and minimax problems.

\lemOccLin*

\begin{proof}
  Let $\occ_\depth$ be an occupancy state at time $\depth$ and
  $\vbeta_\depth$ be a decision rule.
Then the next occupancy state
  $\occ' = \nxt(\occ_\depth,\vbeta_\depth)$ satisfies, for any $s'$
  and $(\vth,\va,\vz)$:
  \begin{align*}
    \occ'(s',(\vth,\va,\vz))
    & \eqdef Pr(s', \vth,\va,\vz | \occ, \beta^1_\depth, \beta^2_\depth) \\
& = \sum_{s\in \cS}
      Pr(s', \vz | s, \va )
      Pr(\va | \vth, \beta^1_\depth, \beta^2_\depth)
      Pr(s, \vth | \occ) \\
    & = \sum_{s\in \cS} \PP{s}{\va}{s'}{\vz} \vbeta_\depth(\vth,\va) \occ_{\depth}(s,\vth) \\
    & = \beta^1_\depth(\theta^1, a^1) \beta^2_\depth(\theta^2, a^2) \sum_{s\in \cS} \PP{s}{\va}{s'}{\vz} \occ_{\depth}(s,\vth).
  \end{align*}
The next occupancy state thus evolves linearly w.r.t.
(i) {\em private} decision rules for a given private history, and
(ii) the occupancy state.
  
  The $1$-Lipschitz-continuity holds because each component of
  $\occ_\depth$ is distributed over multiple components of $\occ'$.
Indeed, let us view two occupancy states as vectors
  $\vx,\vy \in \reals^n$, and their corresponding next states under
  $\vbeta_\depth$ as $M \vx$ and $M \vy$, where
  $M \in \reals^{m\times n}$ is the corresponding transition matrix
  (\ie, which turns $\occ$ into
  $\occ' \eqdef \nxt(\occ_\depth,\vbeta_\depth)$.
Then,
  \begin{align*}
    \norm{M\vx - M\vy}_{1}
& \eqdef \sum_{j=1}^m \ \abs{ \sum_{i=1}^n M_{i,j} (x_i - y_i) } \\
    & \leq \sum_{j=1}^m  \sum_{i=1}^n \abs{M_{i,j} (x_i - y_i) }
    & \text{(convexity of $\abs{\cdot}$)} \\
    & = \sum_{j=1}^m  \sum_{i=1}^n M_{i,j} \abs{ x_i - y_i }
    & \text{($\forall {i,j},\ M_{i,j}\geq 0$)} \\
    & = \sum_{i=1}^n \underbrace{\sum_{j=1}^m M_{i,j}}_{=1} \abs{ x_i - y_i } 
    & \text{($M$ is a transition matrix)} \\
& \eqdef \norm{\vx-\vy}_1.
    & \qedhere
  \end{align*}
\end{proof}

\lemQLC*

\begin{proof}
  As demonstrated in Sec.~\ref{sec|LC|V},
  Corollary~\ref{cor|V|LC|occ}, in finite horizon problems, the
  optimal value function is LC in occupancy space.

  Then, by definition:
  \begin{align*}
    Q^*_\depth(\occ_\depth,\beta^1_\depth,\beta^2_\depth)
    & = r(\occ_\depth, \beta^1_\depth, \beta^2_\depth)
      + \gamma V^*_{\depth+1}(\nxt(\occ_\depth, \beta^1_\depth, \beta^2_\depth)),
  \end{align*}
  where the first term (reward-based) is
$\l_r$-LC (in each $\beta^i_\depth$), with
  $\l_r= \frac{r_{\max} - r_{\min}}{2}$, and
the second term is $(\gamma
  \cdot \l_{H-\depth} \cdot 1)$-LC, with $\l_{H-\depth} =
  \frac{V^{\max}_{H-\depth}-V^{\min}_{H-\depth}}{2}$.
$Q^*_\depth$ is thus $\l^{Q^*}_{H-\depth}$-LC with
  $\l^{Q^*}_{H-\depth} = \l_r + \gamma \cdot \l^V_{H-\depth}$.
\end{proof}

\subsection{Properties of $V^*$}

\subsubsection{Finite-Horizon Lipschitz-Continuity of $V^*$}
\label{proofLemVLinOcc}
\label{proofCorVLCOcc}

The next two results demonstrate that, in the finite horizon setting,
$V^*$ is Lipschitz-continuous (LC) in occupancy space, which allows
defining LC upper- and lower-bounding approximators.

\lemVLinOcc*

\begin{proof}
  This property trivially holds for $\depth=H-1$ because
  \begin{align*}
    V_{H-1}(\occ_{H-1},\vbeta_{H-1:H-1}) 
    & = r(\occ_{H-1},\vbeta_{H-1}) \\
    & = \sum_{s,a^1,a^2} \left( \sum_\vth \occ_{H-1}(s,\vth) \beta^1_{H-1}(\theta^1,a^1) \beta^2_{H-1}(\theta^2,a^2) \right) r(s,a^1, a^2) \\
    & = \sum_{s,\vth} \occ_{H-1}(s,\vth) \left( \sum_{a^1,a^2} \beta^1_{H-1}(\theta^1,a^1) \beta^2_{H-1}(\theta^2,a^2) r(s,a^1, a^2) \right).
  \end{align*}
Now, let us assume that the property holds for
  $\depth+1 \in \{1 \twodots H-1\}$.
Then,
\begin{align*}
    V_\depth(\occ_\depth,\vbeta_{\depth:H-1}) 
    & = \left[ \sum_{s,a^1,a^2} \left( \sum_\vth \occ(s,\vth) \beta^1_\depth(\theta^1,a^1) \beta^2_\depth(\theta^2,a^2) \right) r(s,a^1, a^2) \right] 
    + \gamma V_{\depth+1}\left( \nxt (\occ,\vbeta_\depth), \vbeta_{\depth+1:H-1} \right) \\
    & = \left[ \sum_{s,\vth} \occ(s,\vth) \left( \sum_{a^1,a^2} \beta^1_\depth(\theta^1,a^1) \beta^2_\depth(\theta^2,a^2) r(s,a^1, a^2) \right) \right] 
    + \gamma V_{\depth+1}\left( \nxt (\occ,\vbeta_\depth), \vbeta_{\depth+1:H-1} \right) .
  \end{align*}
  As
  \begin{itemize}
  \item $\nxt(\occ_\depth,\vbeta_\depth)$ is linear in $\occ_\depth$
    {\footnotesize (Lemma~\ref{lem|occ|lin})} and
  \item $V_{\depth+1}(\occ_{\depth+1}, \vbeta_{\depth+1:H-1})$ is
    linear in $\occ_{\depth+1}$ {\footnotesize (induction hypothesis)},
  \end{itemize}
  their composition,
  $V_{\depth+1} ( \nxt (\occ_\depth,\vbeta_\depth),
  \vbeta_{\depth+1:H-1} )$,
  is also linear in $\occ_\depth$, and so is
  $V_\depth(\occ_\depth,\vbeta_{\depth:H-1})$.
\end{proof}

\corVLCOcc*

\begin{proof}
  At depth $\depth$, the value of any behavioral strategy
  $\vbeta_{\depth:H-1}$ is bounded, independently of $\occ_\depth$, by
  \begin{align*}
    V^{\max}_\depth & \eqdef h(H,\gamma,\tau) \cdot \max_{s,\va}r(s,\va) 
    \text{ and } \\
    V^{\min}_\depth &\eqdef h(H,\gamma,\tau) \cdot \min_{s,\va}r(s,\va), \\
    \text{where } h(H,\gamma,\tau) & \eqdef
                                     \begin{cases}
                                       \frac{1-\gamma^{H-\depth}}{1-\gamma} & \text{if } \gamma<1, \\
                                       (H-\depth) & \text{if } \gamma=1. \\
                                     \end{cases}
  \end{align*}
Because the occupancy space at $\depth$ is a probability simplex,
  for any $\occ$ and $\occ'$ in this space,
  $\norm{\occ-\occ'}_{1} \leq 2$.
As a consequence, such a value function being linear in
  $\occ_\depth$ (\cf Lemma~\ref{lem|V|lin|occ}), it is also
  $\l_{H-\depth}$-LC, \ie,
  \begin{align*}
    \abs{V_{\vbeta_{\depth:H-1}}(\occ) -V_{\vbeta_{\depth:H-1}}(\occ')}
    & \leq \l_{H-\depth} \norm{\occ-\occ'}_{\p} \quad (\forall \occ, \occ'), \\
    \text{with }
    \l_{H-\depth}
    & = \frac{V^{\max}_{H-\depth}-V^{\min}_{H-\depth}}{2}.
  \end{align*}

  Considering now optimal solutions, this means that, at depth
  $\depth$ and for any $(\occ,\occ') \in \Occ_\depth$:
  \begin{align*}
    V^*_{\depth}(\occ) - V^*_{\depth}(\occ') 
    & = \max_{\beta^1_{\depth:H-1}} \min_{\beta^2_{\depth:H-1}} V_{\depth}(\occ, \beta^1_{\depth:H-1}, \beta^2_{\depth:H-1}) 
    - \max_{\beta'^1_{\depth:H-1}} \min_{\beta'^2_{\depth:H-1}} V_{\depth}(\occ', \beta'^1_{\depth:H-1}, \beta'^2_{\depth:H-1}) \\
& \leq \max_{\beta^1_{\depth:H-1}} \min_{\beta^2_{\depth:H-1}} \left[ V_{\depth}(\occ', \beta^1_{\depth:H-1}, \beta^2_{\depth:H-1}) + \l_{H-\depth} \norm{\occ-\occ'}_{\p} \right] 
    - \max_{\beta'^1_{\depth:H-1}} \min_{\beta'^2_{\depth:H-1}} V_{\depth}(\occ', \beta'^1_{\depth:H-1}, \beta'^2_{\depth:H-1}) \\
& = \l_{H-\depth} \norm{\occ-\occ'}_{\p}.
  \end{align*}
Symmetrically,
  $V^*_{\depth}(\occ) - V^*_{\depth}(\occ') \geq -\l_{H-\depth}
  \norm{\occ-\occ'}_{\p}$, hence the expected result:
  \begin{align*}
    \abs{ V^*_{\depth}(\occ) - V^*_{\depth}(\occ') } 
    & \leq \l_{H-\depth} \norm{\occ-\occ'}_{\p}.
    \qedhere
  \end{align*}
\end{proof}

\subsection{HSVI for POSGs when $\gamma<1$}
\label{proofLemMaxRadius}
\label{proofLemFiniteTrials}
\label{proofLemThr}

The following results help
(i) tune zs-OMG-HSVI's radius parameter $\rho$, ensuring that
trajectories will always stop,
and (ii) then demonstrate the finite time convergence of this
algorithm.

\lemMaxRadius*

\begin{proof}
  First, we have (for $\depth\geq1$):
  \begin{align*}
    \thr(\depth)
    & = \gamma^{-\depth}\epsilon - \sum_{i=1}^\depth 2 \radius \l \gamma^{-i} \\
& = \gamma^{-\depth}\epsilon - 2 \radius \l \gamma^{-1} \frac{\gamma^{-\depth}-1}{\gamma^{-1}-1} \\
    & = \gamma^{-\depth}\epsilon - 2 \radius \l \frac{\gamma^{-\depth}-1}{1-\gamma}.
  \end{align*}
  Then, let us derive the following equivalent inequalities:
  \begin{align*}
    0
    & < \thr(\depth) \\ 2 \radius \l \frac{\gamma^{-\depth}-1}{1-\gamma}
    & < \gamma^{-\depth}\epsilon \\
\radius
    & < \frac{1}{2\l} \frac{1-\gamma}{\gamma^{-\depth}-1} \gamma^{-\depth} \epsilon \\
\radius
    & < \frac{1}{2\l} \frac{1-\gamma}{1-\gamma^{\depth}}  \epsilon.
\end{align*}
  To ensure positivity of the threshold for any $\depth \geq 1$, one
  thus just needs to set $\radius$ as a positive value smaller than
  $\frac{1}{2\l} \frac{1-\gamma}{1-\gamma^{H}} \epsilon$. \end{proof}

\lemFiniteTrials*

\begin{proof}{(detailed version)}
  Since $W$ is the largest possible width, any trajectory stops in the
  worst case at depth $\depth$ such that
  \begin{align*}
\thr(\depth) & < \WUL \\
\gamma^{-\depth}\epsilon - 2 \radius \l \frac{\gamma^{-\depth}-1}{1-\gamma}
    & < \WUL \qquad \text{(from Eq. (\ref{eq|thr}))} \\
\gamma^{-\depth} \underbrace{\left(\epsilon - \frac{2 \radius \l}{1-\gamma} \right)}_{>0 \quad \text{(Lem.~\ref{lem:MaxRadius})}}
    & < \WUL - \frac{2 \radius \l }{1-\gamma}  \\
\gamma^{-\depth} 
    & < \frac{
      \WUL - \frac{2 \radius \l }{1-\gamma}
    }{
      \epsilon - \frac{2 \radius \l}{1-\gamma}
    }\\
-\depth\ln(\gamma)
    & < \ln\left(\frac{
          \WUL - \frac{2 \radius \l }{1-\gamma}
        }{
          \epsilon - \frac{2 \radius \l}{1-\gamma}
        }\right)\\
\depth 
    & <
    \log_{\gamma}\left(\frac{
        \epsilon - \frac{2 \radius \l}{1-\gamma}
      }{
        \WUL - \frac{2 \radius \l }{1-\gamma}
      }\right). \qedhere
  \end{align*}
\end{proof}

Here is a small preliminary result.

\begin{restatable}{lemma}{lemThrRadius}
  \label{lem|thr|radius}
  For any $\depth \in 1..H-1$,
  \begin{align*}
    \gamma \thr(\depth) = \thr(\depth-1) - 2 \radius \l_{\depth-1}.
  \end{align*}
\end{restatable}

\begin{proof}
Let us first remind the definition of $\thr(\depth)$ from Equation \eqref{eq|def|thr}:
  \begin{align*}
    \thr(\depth)
    & \eqdef \gamma^{-\depth}\epsilon - \sum_{i=1}^\depth 2 \radius \l_{\depth-i} \gamma^{-i}.
      \intertext{We can then write:}
      \thr(\depth-1) - \gamma \cdot \thr(\depth) 
    & =  \left( \gamma^{-(\depth-1)}\epsilon - \sum_{i=1}^{\depth-1} 2 \radius \l_{(\depth-1)-i} \gamma^{-i} \right)
      - \gamma \cdot\left( \gamma^{-\depth} \epsilon - \sum_{i=1}^\depth 2 \radius \l_{\depth-i} \gamma^{-i} \right)\\
    & =  \gamma^{-\depth+1}\epsilon - \sum_{i=1}^{\depth-1} 2 \radius \l_{(\depth-1)-i} \gamma^{-i}
      - \gamma^{-\depth+1}\epsilon + \gamma \sum_{i=1}^\depth 2 \radius \l_{\depth-i} \gamma^{-i} \\
    & = - \sum_{i=1}^{\depth-1} 2 \radius \l_{(\depth-1)-i} \gamma^{-i}
      + \gamma \sum_{i=1}^\depth 2 \radius \l_{\depth-i} \gamma^{-i} \\
    & = - \sum_{i=1}^{\depth-1} 2 \radius \l_{(\depth-1)-i} \gamma^{-i}
      + \sum_{i=1}^\depth 2 \radius \l_{\depth-i} \gamma^{-i+1} \\
    & = - \sum_{i=1}^{\depth-1} 2 \radius \l_{(\depth-1)-i} \gamma^{-i}
      + \sum_{j=0}^{\depth-1} 2 \radius \l_{\depth-(j+1)} \gamma^{-j} \\
    & = 2 \radius \l_{\depth-1}.
  \end{align*}
\end{proof}

\lemThr*

\begin{proof}
The trial terminated at depth $\depth$,
so that $\occ'=\occ_\depth$ (the only occupancy state that can be
  reached from $\occ_{\depth-1}$ when following
  $\beta^{U,1},\beta^{L,2}$ from line \ref{alg|hsvi|zsomg|nextS} of
  Alg.~\ref{alg|HSVI|zsOMG}) must satisfy
  \begin{align*}
    \width(\hat{V}(\occ_\depth)) & \leq \thr(\depth).
  \end{align*}
  Then:
  \begin{align*}
    \width(\cK_{\occ_{\depth-1}}\hat{V}(\occ_{\depth-1}))
& = \width(\cH\hat{V}(\occ_{\depth-1}))
\\
    & \specialcell{
      \ = \cH U(\occ_{\depth-1}) - \cH L(\occ_{\depth-1})
      \hfill \text{(def. of $\width(\cdot)$)}
      } \\
    & \specialcell{
      \ = \nev(\Gamma^{\occ_{\depth-1}}(U)) - \nev(\Gamma^{\occ_{\depth-1}}(L))
      \hfill \text{(def. of $\cH$)}
      } \\
    & = \max_{\beta^1}\min_{\beta^2} Q^{*,U}(\occ_{\depth-1},\beta^1,\beta^2) - \min_{\beta^2} \max_{\beta^1} Q^{*,L}(\occ_{\depth-1},\beta^1,\beta^2) \\
    & = \min_{\beta^2} Q^{*,U}(\occ_{\depth-1},\beta^{U,1},\beta^2) -  \max_{\beta^1} Q^{*,L}(\occ_{\depth-1},\beta^1,\beta^{L,2}) \\
    & \leq Q^{*,U}(\occ_{\depth-1},\beta^{U,1},\beta^{L,2}) -  Q^{*,L}(\occ_{\depth-1},\beta^{U,1},\beta^{L,2}) \\
    & = \gamma (U(\nxt(\occ_{\depth-1},\beta^{U,1},\beta^{L,2})) - L(\nxt(\occ_{\depth-1},\beta^{U,1},\beta^{L,2}))) \\
    & = \gamma \width(\hat{V}(\occ_{\depth})) \\
    & \leq \gamma \thr(\depth) \\
    & \specialcell{
      \ = \thr(\depth-1) - 2 \radius \l_{\depth-1} \hfill \text{(from Lemma \ref{lem|thr|radius})}.
      }
\end{align*}
  This proves the first point.

  Now, the updated approximators $\cK_{\occ_{\depth-1}}U_{\depth-1}$
  and $\cK_{\occ_{\depth-1}}L_{\depth-1}$ are both $\l_{\depth-1}$-LC,
  which immediately gives the second point.
\end{proof}

\subsection{HSVI for POSGs when $\gamma=1$}
\label{proofLemMaxRadiusBis}
\label{proofLemFiniteTrialsBis}
\label{proofLemThrBis}

The following results help
(i) tune zs-OMG-HSVI's radius parameter $\rho$, ensuring that
trajectories will always stop,
and (ii) then demonstrate the finite time convergence of this
algorithm.

\lemMaxRadiusBis*

\begin{proof}
  First, we have (for $\depth\in \{1,\dots,H-1\}$):
  \begin{align*}
    \thr(\depth)
    & \eqdef \epsilon - \sum_{i=1}^\depth 2 \radius \l_{\depth-i} \\
    & = \epsilon - \sum_{i=1}^\depth 2 \radius (H-(\depth-i))\cdot(r_{\max}-r_{\min}) \\
    & = \epsilon - 2 \radius (r_{\max}-r_{\min}) \left[ \depth(H-\depth) + \sum_{i=1}^\depth i \right]  \\
    & = \epsilon - 2 \radius (r_{\max}-r_{\min}) \left[ \depth H - \depth^2 + \frac{1}{2}\depth (\depth+1) \right]  \\
    & = \epsilon - 2 \radius (r_{\max}-r_{\min}) \left[  (H+\frac{1}{2}) \depth - \frac{1}{2} \depth^2 \right]  \\
    & = \epsilon - \radius (r_{\max}-r_{\min}) \left[ (2H+1) \depth - \depth^2 \right]  \\
    & = \epsilon - \radius (r_{\max}-r_{\min}) \left[ (2H + 1 - \depth) \depth \right].
  \end{align*}
  Then, let us derive the following equivalent inequalities:
  \begin{align*}
    0
    & < \thr(\depth) \\ 
    \radius (r_{\max}-r_{\min}) (2H + 1 - \depth) \depth
    & < \epsilon
    & \text{(holds when $\depth=0$ and $\depth=H+1$)}
    \\
    \radius 
    & < \frac{\epsilon}{(r_{\max}-r_{\min}) (2H + 1 - \depth) \depth}
    & \text{(when $\depth \in \{0,\dots,H+1\}$).}
  \end{align*}
  The function
  $f: \depth \mapsto \frac{\epsilon}{(r_{\max}-r_{\min}) (2H + 1 -
    \depth) \depth}$
  reaches its minimum (for $\depth \in (0,H+1)$) when
  $\depth=H+\frac{1}{2}$.
To ensure positivity of the threshold for any $\depth \in \{1, \dots, H-1 \}$, one
  thus just needs to set $\radius$ as a positive value smaller than
  $\frac{\epsilon}{(r_{\max}-r_{\min}) (H + 1)H}$.
\end{proof}

\lemThrBis*

\begin{proof}
  Observe that from the definition of the sequence $\thr(\depth)$ in
  Equation \eqref{eq|def|thr} it follows that
  \begin{align*}
    \gamma \thr(\depth) = \thr(\depth-1) - 2 \radius \l_\depth.
  \end{align*}
  Moreover, the trial terminated at depth $\depth$.
Therefore, $\occ'=\occ_\depth$ (the only occupancy state that can be
  reached from $\occ_{\depth-1}$ when following
  $\beta^{U,1},\beta^{L,2}$ from line \ref{alg|hsvi|zsomg|nextS} of
  Alg.~\ref{alg|HSVI|zsOMG}) must satisfy
  \begin{align*}
    \width(\hat{V}(\occ_\depth)) & \leq \thr(\depth).
  \end{align*}
  Then:
  \begin{align*}
    \width(\cK_{\occ_{\depth-1}}\hat{V}(\occ_{\depth-1}))
& = \width(\cH\hat{V}(\occ_{\depth-1}))
\\
    & \specialcell{
      \ = \cH U(\occ_{\depth-1}) - \cH L(\occ_{\depth-1})
      \hfill \text{(def. of $\width(\cdot)$)}
      } \\
    & \specialcell{
      \ = \nev(\Gamma^{\occ_{\depth-1}}(U)) - \nev(\Gamma^{\occ_{\depth-1}}(L))
      \hfill \text{(def. of $\cH$)}
      } \\
    & = \max_{\beta^1}\min_{\beta^2} Q^{*,U}(\occ_{\depth-1},\beta^1,\beta^2) - \min_{\beta^2} \max_{\beta^1} Q^{*,L}(\occ_{\depth-1},\beta^1,\beta^2) \\
    & = \min_{\beta^2} Q^{*,U}(\occ_{\depth-1},\beta^{U,1},\beta^2) -  \max_{\beta^1} Q^{*,L}(\occ_{\depth-1},\beta^1,\beta^{L,2}) \\
    & \leq Q^{*,U}(\occ_{\depth-1},\beta^{U,1},\beta^{L,2}) -  Q^{*,L}(\occ_{\depth-1},\beta^{U,1},\beta^{L,2}) \\
    & = \gamma (U(\nxt(\occ_{\depth-1},\beta^{U,1},\beta^{L,2})) - L(\nxt(\occ_{\depth-1},\beta^{U,1},\beta^{L,2}))) \\
    & = \gamma \width(\hat{V}(\occ_{\depth})) \\
    & \leq \gamma \thr(\depth) \\
    & = \thr(\depth-1) - 2 \radius \l_\depth.
\end{align*}
  This proves the first point.

  Now, the updated approximators $\cK_{\occ_{\depth-1}}U_{\depth-1}$
  and $\cK_{\occ_{\depth-1}}L_{\depth-1}$ are both $\l_\depth$-LC,
  which immediately gives the second point.
\end{proof}

 }{}

\end{document}